\documentclass{article}

\def\final{1}

\ifx\preprint\undefined
    \ifx\final\undefined
    \usepackage{neurips_2020}
    \else
    \usepackage[final]{neurips_2020}
    \fi
\else
\usepackage[preprint]{neurips_2020}
\fi

\setcitestyle{numbers,square}

\usepackage[utf8]{inputenc} % allow utf-8 input
\usepackage[T1]{fontenc}    % use 8-bit T1 fonts
\usepackage{hyperref}       % hyperlinks
\usepackage{url}            % simple URL typesetting
\usepackage{booktabs}       % professional-quality tables
\usepackage{amsfonts}       % blackboard math symbols
\usepackage{nicefrac}       % compact symbols for 1/2, etc.
\usepackage{microtype}      % microtypography
\usepackage{hyperref}
\usepackage[pdftex]{graphicx}

\usepackage[dvipsnames]{xcolor}
\hypersetup{colorlinks,linkcolor={blue},citecolor={blue},urlcolor={red}}

\usepackage{amsfonts}
\usepackage{amsmath}
\usepackage{amstext}
\usepackage{amssymb}
\usepackage[thmmarks, amsmath, thref, amsthm]{ntheorem}
\usepackage{lipsum} 

\usepackage{float}
\usepackage{subfig}
\usepackage[export]{adjustbox}
\usepackage{listings}
\usepackage{changepage}

\usepackage{wrapfig}

\usepackage{graphicx}

\newcommand{\lref}[2]{\hyperref[#1]{ \ref{#1}#2}}

\usepackage{xcolor}
\usepackage{tcolorbox}
\usepackage[font=small,labelfont=bf]{caption}

\usepackage{multirow}

\usepackage{titlesec}
\titlespacing*{\section}{0pt}{0ex plus 0ex minus 0ex}{0ex plus 0ex}
\titlespacing*{\subsection}{0pt}{0ex plus 0ex minus 0ex}{0ex plus 0ex}
\newfloat{lstfloat}{htbp}{lop}
\floatname{lstfloat}{Listing}

\newtheorem{theorem}{Theorem}[section]
\newtheorem{lemma}[theorem]{Lemma}
\newtheorem{definition}{Definition}[section]

\newtheorem{example}{Example}[section]

\usepackage[ruled,longend]{algorithm2e}
\usepackage{setspace}
\SetKwRepeat{Do}{do}{while}

\title{Log-Likelihood Ratio Minimizing Flows: \\ Towards Robust and Quantifiable \\ Neural Distribution Alignment}

\author{%
  Ben Usman $^{1,2}$ \\
  \texttt{usmn@bu.edu} \\
  \And
  Avneesh Sud $^{2}$ \\
  \texttt{asud@google.com} \\
  \And
  Nick Dufour $^{2}$\\
  \texttt{ndufour@google.com} \\
  \And
  Kate Saenko $^{1,3}$\\
  \texttt{saenko@bu.edu} \\
  \and
  {\normalfont Boston University $^{1}$ \quad Google AI $^{2}$ \quad MIT-IBM Watson AI Lab $^{3}$} \vspace{-10px}
}

\begin{document}

\maketitle

\begin{abstract}
Distribution alignment has many applications in deep learning, including domain adaptation and unsupervised image-to-image translation. Most prior work on unsupervised distribution alignment relies either on minimizing simple non-parametric statistical distances such as maximum mean discrepancy or on adversarial alignment. However, the former fails to capture the structure of complex real-world distributions, while the latter is difficult to train and does not provide any universal convergence guarantees or automatic quantitative validation procedures. In this paper, we propose a new distribution alignment method based on a log-likelihood ratio statistic and normalizing flows. We show that, under certain assumptions, this combination yields a deep neural likelihood-based minimization objective that attains a known lower bound upon convergence. We experimentally verify that minimizing the resulting objective results in domain alignment that preserves the local structure of input domains.
\end{abstract}
\section{Introduction}

% The primary goal of domain adaptation is to improve the model performance on the test dataset if test and train sets are sampled from different distributions. 
% Early approaches to domain adaptation include instance reweighting [] and subspace alignment [], but t
% \textcolor{red}{need sentence introducting the task: Unsupervised domain alignment methods learn a representation of two input datasets that aligns their statistical distributions. For example, they can learn to map digits from a ``handwritten'' to a ``typed'' domain without pairwise correspondence labels.} 

The goal of unsupervised domain alignment is to find a transformation of one dataset that makes it similar to another dataset while preserving the structure of the original. The majority of modern neural approaches to domain alignment directly search for a transformation of the dataset that minimizes an empirical estimate of some statistical distance - a non-negative quantity that takes lower values as datasets become more similar. The variability of what ``similar'' means in this context, which transformations are allowed, and whether data points themselves or their feature representations are aligned, leads to a variety of domain alignment methods. Unfortunately, existing estimators of statistical distances either restrict the notion of similarity to enable closed-form estimation \cite{sun2016deep}, or rely on adversarial (min-max) training \cite{tzeng2017adversarial} that makes it very difficult to quantitatively reason about the performance of such methods \cite{barratt2018note,borji2019pros,theis2015note}. In particular, the value of the optimized adversarial objective conveys very little about the quality of the alignment, which makes it difficult to perform automatic model selection on a new dataset pair. On the other hand, Normalizing Flows \cite{rezende2015variational} are an emerging class of deep neural density models that do not rely on adversarial training. They model a given dataset as a random variable with a simple known distribution transformed by an unknown invertible transformation parameterized using a deep neural network. Recent work on normalizing flows for maximum likelihood density estimation made great strides in defining new rich parameterizations for these invertible transforms \cite{dinh2016density,grathwohl2018ffjord,kingma2018glow}, but little work focused on flow-based density alignment \cite{grover2019alignflow,yang2019pointflow}.

In this paper, we present the Log-likelihood Ratio Minimizing Flow (LRMF), a new non-adversarial approach for aligning distributions in a way that makes them indistinguishable for a given family of density models $M$. 
% It uses unique properties of normalizing flows to turn this adversarial optimization problem into a minimization problem. 
We consider datasets $A$ and $B$ indistinguishable with respect to the family $M$ if there is a single density model in $M$ that is optimal for both $A$ and $B$ individually since in this case there is no way of telling which of two datasets was used for training it. For example, two different distributions with the same means and covariances are indistinguishable for the Gaussian family $M$ since we can not tell which of two datasets was used by examining the model fitted to either one of them.
For a general $M$, we can quantitatively measure whether 
% such model 
% optimal for both dataset individually 
% exists
two datasets are indistinguishable by models from $M$ by comparing average log-likelihoods of two ``private'' density models each fit independently to $A$ and $B$, to the average log-likelihood of the ``shared'' model fit to both datasets at the same time. We observe that, if datasets are sufficiently large, the maximum likelihood of the ``shared'' model would reach the likelihoods of two ``private'' models on respective datasets only if the shared model is optimal for both of them individually, and consequently datasets are equivalent with respect to $M$. Then a density model optimal for $A$ is guaranteed to be optimal for $B$ and vice versa.

% , and then fit another ``shared'' model $P_S$ to the combined dataset containing samples from both $A$ and $B$, the average log-likelihood scores of the shared model $P_S$ on $A$ and $B$ would match average log-likelihoods of $P_A$ on $A$ and $P_B$ on $B$ only if $P_S$ was an optimal density model for both $A$ and $B$ independently, i.e. these datasets are equivalent w.r.t. $M$. 
We want to find a transformation $T(x)$ that transforms dataset $A$ in a way that makes the transformed dataset $T(A)$ equivalent to $B$ for the given family $M$. We do that by minimizing the aforementioned gap between average log-likelihood scores of ``shared'' and ``private'' models. 
% on $A$ and $B$.
% , by ensuring that the optimal ``shared'' model for $A'$ and $B$ is also optimal for both $A'$ and $B$ independently. 
% For example, if $M$ was the family of normal densities, and we had two datasets with known moments, then the linear transformation that matched their moments would make them ``equivalent'' w.r.t. the normal family. 
In this paper, we show that, generally, such $T(x)$ can be found only by solving a min-max optimization problem, but if $T(x, \phi)$ is a family of normalizing flows, then the flow $T(x, \phi^*)$ that makes $T(A, \phi^*)$ and $B$ equivalent with respect to $M$ can be found by minimizing a single objective that attains zero upon convergence. This enables automatic model validation and hyperparameter tuning on the held-out set.

To sum up, the novel non-adversarial data alignment method presented in this paper combines the clear convergence criteria found in non-parametric and simple parametric approaches and the power of deep neural discriminators used in adversarial models. Our method finds a transformation of one dataset that makes it ``equivalent'' to another dataset with respect to the specified family of density models. We show that if that transformation is restricted to a normalizing flow, the resulting problem can be solved by minimizing a single simple objective that attains zero only if two domains are correctly aligned. We experimentally verify this claim and show that the proposed method preserves the local structure of the transformed distribution and that it is robust to model misspecification by both over- and under-parameterization. We show that minimizing the proposed objective is equivalent to training a particular adversarial network, but in contrast with adversarial methods, the performance of our model can be inferred from the objective value alone. We also characterize the the vanishing of generator gradient mode that our model shares with its adversarial counterparts, and principal ways of detecting it.
\section{Log-Likelihood Ratio Minimizing Flow} 
% \textcolor{red}{use name of method}}

\begin{figure*}
\vspace{-45px}
\begin{center}
\includegraphics[trim=10 235 80 0,width=0.9\textwidth,draft=false]{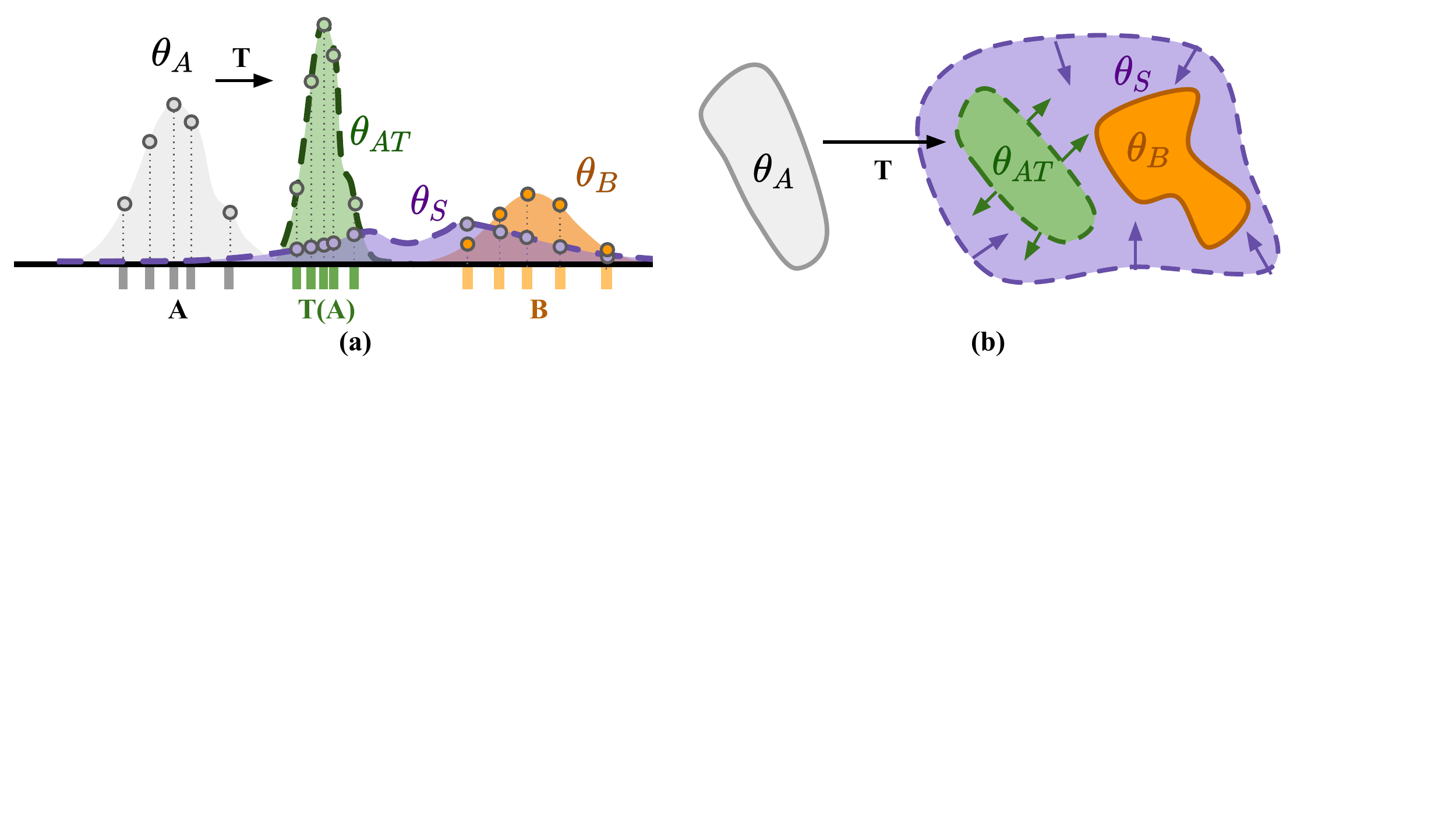}

\caption{To align input datasets $A$ and $B$, we look for a transformation $T$ that makes $T(A)$ and $B$ ``indistinguishable''. \textbf{(a)} We propose the log-likelihood ratio distance $d_{\Lambda}(T(A), B)$ that compares likelihoods of density models $\theta_{AT}$ fitted to $T(A)$ and $\theta_B$ to $B$ independently with the likelihood of $\theta_S$ optimal for the combined dataset $T(A) \cup B$. This problem is adversarial, but we show how to reduce it to minimization if $T$ is a normalizing flow. \textbf{(b)} Colored contours represent level sets of models for $B$ (orange), $T(A)$ (green), and $\theta_S$ (purple), contour sizes corresponds to entropies of these models. Only $\theta_{AT}$ and $\theta_S$ (dashed) change during training. The proposed objective can be viewed as maximizing the entropy of the transformed dataset, while minimizing the combined entropy of $T(A) \cup B$, i.e. expanding green contour while squeezing the purple contour around green and orange contours. At equilibrium, $\theta_S$ and $\theta_{AT}$ model the same distribution as $\theta_B$, i.e. shapes of purple and green contours match and tightly envelope the orange contour. \textit{Best viewed in color}. \vspace{-27px}}
\label{fig:fig1_v3}
\end{center}
\end{figure*}

% \begin{wrapfigure}{r}{0.5\textwidth}
% \begin{center}
% \includegraphics[trim=0 270 360 10,width=0.45\textwidth]{imgs/figs/fig3_narrow.pdf}
% \caption{If we train two normalizing flows: $G$ on vertices of the mesh $A$ and $F$ on vertices of $B$, the marginal vertex distribution of $F(G^{-1}(A))$ matches $B$, but the local structure of the original manifold is distorted. On the other hand, training a separate flow $T$ to match $T(A)$ with $B$ using the likelihood objective preserves the local structure.
% % \vspace{-20px}
% }
% \label{fig:mesh_f_g_inv}
% \end{center}
% \end{wrapfigure}

In this section, we formally define the proposed method for aligning distributions. We assume that $M(\theta)$ is a family of densities parameterized by a real-valued vector $\theta$, and we fit models from $M$ to data by maximizing its likelihood across models from $M$.
% , or equivalently by minimizing the KL divergence with corresponding empirical Dirac delta mixture ``densities''
% We introduce the log-likelihood ratio pseudo-distance and show its relation to the test statistic of the same name. 
Intuitively, if we fit two models $\theta_A$ and $\theta_B$ to datasets $A$ and $B$ independently, and also fit a single shared model $\theta_S$ to the combined dataset $A \cup B$, then the log-likelihood ratio distance would equal the difference between the log-likelihood of that optimal ``shared'' and the two optimal ``private'' models (Definition~\ref{def:llr_distance}).
% \todo{rewrite next paragraph and related to the new figure} Figure illustrates how this quantity changes as two datasets become more similar: the log-likelihood ratio distance between $A'$ and $B$ [\textbf{(3)+(4)-(2)-(5)}] is larger then the log-likelihood ratio distance between $B$ and $C$ [\textbf{(6)+(7)-(8)-(5)}] because the shared approximation $\theta_{BC}$ for $B$ and $C$ approximates them almost as well as their private approximations $\theta_B$ and $\theta_C$. 
Next, we consider the problem of finding a transformation that would minimize this distance. In general, this would require solving an adversarial optimization problem (\ref{eq:llrd_with_transform}), but we show that if the transformation is restricted to the family of normalizing flows, then the optimal one can be found by minimizing a simple non-adversarial objective (Theorem~\ref{th:lrmf_bound}). % We do that by eliminating the optimization over the likelihood of the transformed distribution normalizing flows provide closed-form expression for the likelihood of the distribution . 
We also illustrate this result with an example that can be solved analytically: we show that minimizing the proposed distance between two random variables with respect to the normal density family is equivalent to directly matching their first two moments (Example~\ref{ex:gaussian_lrmf}). Finally, we show the relation between the proposed objective and Jensen-Shannon divergence and show that minimizing the proposed objective is equivalent to training a generative adversarial network with a particular choice of the discriminator family.

% \subsection{Setup}
\textbf{Notation.} Let $\log P_M(X; \theta) := \mathbb{E}_{x \sim P_X} \log P_M(x; \theta)$ denote the negative cross-entropy between the distribution $P_X$ of the dataset $X$ defined over $\mathcal X \subset \mathbb R^n$, and a member $P_M(x; \theta)$ of the parametric family of distributions $M(\theta)$ defined over the same domain, i.e. the likelihood of $X$ given $P_M(x; \theta)$.

\begin{definition}[LRD]\label{def:llr_distance}
Let us define the log-likelihood ratio distance $d_{\Lambda}$ between datasets $A$ and $B$ from $\mathcal X$ with respect to the family of densities $M$, as the difference between log-likelihoods of $A$ and $B$ given optimal models with ``private'' parameters $\theta_A$ and $\theta_B$, and ``shared'' parameters $\theta_S$:

% \begin{align*}
%     d_{\Lambda}(A, B; M) = \min_{\theta_S} & \Big[ \DKL[A ,  M(\theta_S)] + \DKL[B, M(\theta_S)] \Big] \\
%     - \min_{\theta_A} & \DKL[A, M(\theta_A)] - \min_{\theta_B} \DKL[B,  M(\theta_B)] \\
%     = \min_{\theta_S} \max_{\theta_A; \theta_B} & \Big[ \DKL[A ,  M(\theta_S)] + \DKL[B, M(\theta_S)] \\
%     - & \DKL[A, M(\theta_A)] - \DKL[B,  M(\theta_B)] \Big]
% \end{align*}
% \adjustwidthheight{-5pt}{-35pt}{-15px}{-10px}{
% \scalebox{0.5}{
% \hspace*{-6px}
% \fbox{
\vspace{-5px}
\scalebox{0.96}{\parbox{1\linewidth}{%
\begin{align*} % \label{eq:llrd}
\begin{split}
    d_{\Lambda}(A, B; M) 
    % & = \max_{\theta_A} \log P_M(A; \theta_A) + \max_{\theta_B} \log P_M(B; \theta_B) - \max_{\theta_S} \Big[ \log P_M(A; \theta_S) + \log P_M(B; \theta_S) \Big] \\
    & = \max_{\theta_A; \theta_B} \Big[\log P_M(A; \theta_A) + \log  P_M(B; \theta_B) \Big] - \max_{\theta_S} \Big[ \log P_M(A; \theta_S) + \log P_M(B; \theta_S) \Big] \\
    & = \min_{\theta_S} \max_{\theta_A; \theta_B} \Big[\big(\log P_M(A; \theta_A) - \log  P_M(A; \theta_S) \big) + \big(\log  P_M(B; \theta_B)  - \log P_M(B; \theta_S) \big) \Big].
\end{split}
\end{align*}
}}
\vspace{-6px}
\end{definition}
The expression above is also the log-likelihood ratio test statistic $\log \Lambda_n$ for the null hypothesis $H_0 : \theta_A = \theta_B$ for the model described by the likelihood function
$ P(A, B \ | \ \theta_A, \theta_B) = \big[ P_M(A; \theta_A) \cdot P_M(B; \theta_B) \big]$
and intuitively equals to the amount of likelihood we ``lose'' by forcing $\theta_A = \theta_B$ onto the model fitted to approximate $A$ and $B$ independently. Figure~\ref{fig:fig1_v3} illustrates that, in terms of average likelihood, the shared model (purple) is always inferior to two private models from the same class, unless two datasets are in fact just different samples from the same distribution. 

\begin{lemma}\label{lem:llrd}
The log-likelihood ratio distance is non-negative,
% $$ d_{\Lambda}(A, B; M) \geq 0 $$ 
and the equals zero only if there exists a single ``shared'' model that approximates datasets as well as their ``private'' optimal models:

\vspace{-10px}\scalebox{0.95}{\parbox{\linewidth}{%
\begin{gather*}
% d_{\Lambda}(A, B; M) \geq 0 \\
    d_{\Lambda}(A, B; M) = 0 \Leftrightarrow \exists \ \theta_S  : \log P_M(A; \theta_S) = \max_{\theta} \log P_M(A; \theta) \wedge \log P(B; \theta_S) = \max_{\theta} \log P_M(B; \theta).
\end{gather*}
}}
\end{lemma}
\begin{proof} Follows from the fact that the shared part in the Definition \ref{def:llr_distance} is identical to the private part but over a smaller feasibility set $\{\theta_A = \theta_B\}$. See the supplementary Section \ref{sec:lemma21_proof} for the formal proof.\end{proof}

% the minimum over a subset of the feasibility set is always worse then the minimum over the entire set $B \subset A \Rightarrow \min_{x \in B} f(x) \geq \min_{x \in A} f(x)$ and that minimization of two models with shared parameters is strictly more constrained then optimization of two private models. To show the second statement, assume that $d_{\Lambda} = 0$ and the optimal shared model is $\theta_S$, but $\exists \ \theta_A \ L(A; \theta_A) < L(A; \theta_S)$. Since $d_{\Lambda} = 0$, it means that for the optimal private $\theta_B$ holds $L(B; \theta_B) > L(A; \theta_S)$, because otherwise it won't sum to zero. The latter contradicts with the fact that $\theta_B$ is optimal.

% The first statement follows from the fact that the minimum over a subset of the feasibility set is always worse then the minimum over the entire set $B \subset A \Rightarrow \min_{x \in B} f(x) \geq \min_{x \in A} f(x)$ and that minimization of two models with shared parameters is strictly more constrained then optimization of two private models. To show the second statement, assume that $d_{\Lambda} = 0$ and the optimal shared model is $\theta_S$, but $\exists \ \theta_A \ L(A; \theta_A) < L(A; \theta_S)$. Since $d_{\Lambda} = 0$, it means that for the optimal private $\theta_B$ holds $L(B; \theta_B) > L(A; \theta_S)$, because otherwise it won't sum to zero. The latter contradicts with the fact that $\theta_B$ is optimal.

\textbf{Adversarial formulation.} If we introduce the parametric family of transformations $T(x, \phi)$ and try to find $\phi$ that minimizes the log-likelihood ratio distance $\min_{\phi} d_{\Lambda}(T(A; \phi), B; M)$, an adversarial problem arises. Note that for a fixed dataset $B$, only the first term is adversarial, and only w.r.t. $\theta_{AT}$:
\vspace{-8pt}

% \hspace*{-6px}
% \fbox{
\scalebox{0.96}{\parbox{1.05\linewidth}{%
% \adjustwidthheight{-5pt}{-35pt}{-15px}{-10px}{
\begin{gather}\label{eq:llrd_with_transform}
    \begin{aligned}
    \min_{\phi, \theta_S} \max_{\theta_{AT}; \theta_B} \Big[\log P_M(T(A; \phi); \theta_{AT}) + \log  P_M(B; \theta_B) - \log  P_M(T(A; \phi); \theta_S) - \log P_M(B; \theta_S) \Big]
    \end{aligned}
\end{gather}
% }
}}
Figure~\lref{fig:fig1_v3}{b} illustrates that minimizing this objective (\ref{eq:llrd_with_transform}) over $\theta_S$ while maximizing it over $\theta_{AT}$ corresponds to minimizing entropy (``squeezing'') of the combination of $T(A)$ and $B$ while maximizing entropy of (``expanding'') transformed dataset $T(A)$ as much as possible.

% But first, let us examine the intuition on how different components of the objective above would have affected the transformed dataset if we optimized them directly. The model $\theta_{AT}$ in the definition above stands for the optimal approximation of $T(A, \phi)$. Figure~\ref{fig:llrd_adv_min} shows that, when minimized over $\phi$, the first component of the loss, namely $L(T(A, \phi); \theta_S)$, pulls the transformed dataset $A'$ towards the best shared model $\theta_S$ and the third component $- L(T(A, \phi); \theta_{AT})$ pushes it away from its optimal approximation on $M$, therefore ensuring that $A'$ becomes more similar to $B$ with respect to $M$ without collapsing onto $M$. For example, if two datasets were already equivalent w.r.t. $M$, but not necessarily on $M$ itself, the objective above would \textit{not} encourage them to move towards $M$. When optimized over $\theta_{S}$ and $\theta_{AT}$, the first three components of the objective above ensure that $\theta_{AT}$ is still an optimal approximation of $A'$ and $\theta_S$ is still an optimal approximation of $A'$ and $B$ combined. The last component of the loss optimized over $\theta_B$ is a constant that can be computed separately. 

% \input{figs/fig2.tex}

\textbf{Non-adversarial formulation.} The adversarial objective (\ref{eq:llrd_with_transform}) requires finding a new optimal model $\theta_{AT}$ for each new value of $\phi$ to find the maximal likelihood of the transformed dataset $T(A)$, but Figure~\lref{fig:fig1_v3}{a} illustrates that the likelihood of the transformed dataset can be often estimated from the parameters of the transformation $T$ alone. For example, if $T$ uniformly squeezes the dataset by a factor of two, the average maximum likelihood of the transformed dataset $\max_{\theta} \log P_M(T(A); \theta)$ doubles compared to the likelihood of the original $A$. In general, the likelihood of the transformed dataset is inversely proportional to the Jacobian of the determinant of the applied transformation. The lemma presented below formalizes this relation taking into account the limited capacity of $M$, and leads us to our main contribution: the optimal transformation can be found by simply minimizing a modified version of the objective (\ref{eq:llrd_with_transform}) using an iterative method of one's choice.

\begin{lemma}\label{lem:lrmf}
If $T(x; \phi)$ is a normalizing flow, then the first term in the objective (\ref{eq:llrd_with_transform}) can be bounded in closed form as a function of $\phi$ up to an approximation error $\mathcal E_{bias}$. The equality in (\ref{eq:lrmf_ineq}) holds when the approximation term vanishes, i.e. if $M$ approximates both $A$ and $T(A; \phi)$ equally well; $P_A$ is the true distribution of $A$ and $T[P_A, \phi]$ is the push-forward distribution of the transformed dataset.
% \begin{equation}\label{eq:lrmf_ineq}
% \begin{gathered}
\begin{gather}\label{eq:lrmf_ineq}
\max_{\theta_{AT}} \log P_M(T(A; \phi); \theta_{AT}) \leq \max_{\theta_{A}} \log P_M(A; \theta_{A}) - \log\det|\nabla_x T(A;\phi)| + \mathcal E_{bias}(A, T, M) \\
\mathcal E_{bias}(A, T, M) \triangleq \max_{\phi}\left[\min_{\theta} \mathcal D_{KL}(P_A; M( \theta)) - \min_{\theta} \mathcal D_{KL}(T[P_A, \phi]; M(\theta))\right] \nonumber
\end{gather}
% \end{gathered}
% \end{equation}
\end{lemma}
\begin{proof}
We expand likelihoods of combined and shared datasets given best models from $M$ into respective ``true'' negative entropies and the approximation errors due to the choice of $M$ (KL-divergence between true distributions and their KL-projections onto $M$). Then we replace the entropy of the transformed dataset with the entropy of the original and the log-determinant of the Jacobian of the applied transformation, noting that $\log\det|\nabla_x T^{-1}(T(A, \phi), \phi)| = \log\det|\nabla_x T(A, \phi)|$. We refer readers to the Section \ref{sec:lemma21_proof} of the supplementary for the full proof.
\end{proof}

By applying this lemma to the objective (\ref{eq:llrd_with_transform}) and grouping together terms that do not depend on $\theta_S$ and $\phi$, we finally obtain the final objective.

% \begin{definition}
% The log-likelihood ratio flow bound 
% \begin{align*}
%     {\tilde d}_{\Lambda}(T(A, \phi), B; M) = \min_{\theta_S} \big[ & L(T(A, \phi); \theta_S) + L(B; \theta_S) \\ & + l_T(T(A, \phi), \phi) - c(A, B) \big]
% \end{align*}
% .
% \end{definition}

\begin{definition}[LMRF]
Let us define the log-likelihood ratio minimizing flow (LRMF) for a pair of datasets $A$ and $B$ on $\mathcal X$, the family of densities $M(\theta)$ on $\mathcal X$, and the parametric family of normalizing flows $T(x; \phi)$ from $\mathcal X$ onto itself,
% with likelihood $P_T(x, \phi)$ 
% and an arbitrary choice of prior $P_Z(x)$, 
as the flow $T(x; \phi^*)$ that minimizes $\mathcal L_{\text{LRMF}}$ (\ref{eq:lrmf_loss}),
% \begin{equation}\label{eq:lrmf_loss}
% \begin{aligned}
%     \min_{\phi} \min_{\theta_S} \big[ & L(T(A, \phi); \theta_S) + L(B; \theta_S) \\ & + \log P_T(T(A, \phi), \phi) - c(A, B) \big].
% \end{aligned}    
% \end{equation}
% \adjustwidthheight{-5pt}{-25pt}{-15px}{-10px}{
% }
where the constant
$ c(A, B)$
does not depend on $\theta_S$ and $\phi$, and can be precomputed in advance.
\hspace*{-10px}
\scalebox{0.95}{\parbox{1.1\linewidth}{%
% \begin{equation}\label{eq:lrmf_loss}
% \begin{gathered}
\begin{gather}\label{eq:lrmf_loss}
% \phi^*(A, B; T, M) = \arg\min_{\phi} \min_{\theta_S} \mathcal L_{\text{LRMF}}(A, B, \phi, \theta_S) , \\
\mathcal L_{\text{LRMF}}(A, B, \phi, \theta_S) = -\log\det|\nabla_x T(A; \phi)| - \log P_M(T(A; \phi); \theta_S) - \log P_M(B; \theta_S) + c(A, B), \\
c(A, B) = \max_{\theta_A} \log P_M(A; \theta_A) + \max_{\theta_B} \log P_M(B; \theta_B) \nonumber
\end{gather}
% \end{gathered}    
% \end{equation}
}}

\end{definition}

\begin{theorem}\label{th:lrmf_bound}
If $T(x, \phi)$ is a normalizing flow, then the adversarial log-likelihood ratio distance (\ref{eq:llrd_with_transform}) between the transformed source and target datasets can be bounded via the non-adversarial LRMF objective (\ref{eq:lrmf_loss}), and therefore the parameters of the normalizing flow $\phi$ that make $T(A, \phi)$ and $B$ equivalent with respect to $M$ can be found by minimizing the LRMF objective (\ref{eq:lrmf_loss}) using gradient descent iterations with known convergence guarantees.
\begin{equation}\label{eq:lrmf_bound}
   0 \leq d_{\Lambda}(T(A, \phi), B; M) \leq \mathcal \min_{\theta} L_{\text{LRMF}}(A, B, \phi, \theta) + \mathcal E_{bias}. 
\end{equation}

\end{theorem}

This theorem follows from the definition of $d_{\Lambda}$ and two lemmas provided above that show that the optimization over $\theta_{AT}$ can be (up to the error term) replaced by a closed-form expression for the likelihood of the transformed dataset if the transformation is a normalizing flow. Intuitively, the LRMF loss (\ref{eq:lrmf_loss}) encourages the transformation $T$ to draw all points from $A$ towards the mode of the shared model $P(x, \theta_S)$ via the second term, while simultaneously encouraging $T$ to expand as much as possible via the first term as illustrated in Figure~\lref{fig:fig1_v3}{b}. The delicate balance is attained only when two distributions are aligned, as shown in Lemma \ref{lem:llrd}. The inequality (\ref{eq:lrmf_bound}) is tight (equality holds) only when the bias term is zero, and the shared model is optimal.

% Intuitively, the reason why we were able to replace the adversarial problem (\ref{eq:llrd_with_transform}) with a minimization problem (\ref{eq:lrmf_loss}) is that the third term in Eq~\ref{eq:llrd_with_transform} acts as an entropy-based ``discriminator'', but the flow family provides a closed-form expression for estimating the extra entropy the transformation induces upon the dataset without solving an optimization problem.

The example below shows that the affine log-likelihood ratio minimizing flow between two univariate random variables with respect to the normal density family $M$ corresponds to shifting and scaling one variable to match two first moments of the other, which agrees with our intuitive understanding of what it means to make two distributions ``indistinguishable'' for  the Gaussian family. 
% In later parts of this paper we show experimentally that our method works with arbitrary densities $M$ and normalizing flows $T(x, \phi)$.

\begin{tcolorbox}
\begin{example}\label{ex:gaussian_lrmf}
{\small
Let us consider two univariate normal random variables $A, B$ with moments $\mu_A, \mu_B, \sigma^2_A, \sigma^2_B$, restrict $M$ to normal densities, and the transform $T(x; \phi)$ to the affine family: $T(x; a, b) = ax + b$, i.e. $\theta = (\mu, \sigma)$ and $\phi = (a, b)$. Using the expression for the maximum log-likelihood (negative entropy) of the normal distribution, and the expression for variance of the equal mixture, we can solve the optimization over $\theta_S = (\mu_S, \sigma_S)$ analytically:
\scalebox{0.98}{\parbox{1\linewidth}{%
\begin{gather*}
    \min_{\mu, \sigma} \mathbb E_{X} \log P(X; \mu, \sigma) = - \frac{1}{2}\log(2\pi e \sigma_X^2) = - \log \sigma_X + C \\
    \min_{\theta_{S}} \Big[ - \log P_M(T(A; \phi); \theta_{S}) - \log P_M(B; \theta_{S}) \Big]
    = \log\Big(\frac{1}{2}(a^2 \sigma_A^2 + \sigma_B^2) - \frac{1}{4}(\mu_A + b - \mu_B)^2 \Big) - 2C.
\end{gather*}
}}
Combining expressions above gives us the final objective that can be solved analytically by setting the derivatives with respect to $a$ and $b$ to zero:
\begin{gather*}
\log\det|\nabla_x T(A; \phi)| = \log a \quad \text{ and } \quad c(A, B) = - \log \sigma_A - \log \sigma_B + 2C, \\
\mathcal L_{\text{LRMF}} = - \log a + \log\Big(\frac{1}{2}(a^2 \sigma_A^2 + \sigma_B^2) - \frac{1}{4}(\mu_A + b - \mu_B)^2 \Big) - \log \sigma_A - \log \sigma_B \\
a^* = \frac{\sigma_{B}}{\sigma_A}, \ \ \ b^* = \mu_B - \mu_A .
\end{gather*}
The error term $\mathcal E_{bias}$ equals zero because any affine transformation of a Gaussian is still a Gaussian.
}
\end{example}
\end{tcolorbox}

\textbf{Relation to Jensen-Shannon divergence and GANs.} From the same expansion as in the proof of Lemma \ref{lem:lrmf} and the information-theoretic definition of the Jensen-Shannon divergence (JSD) as the difference between entropies of individual distributions and their equal mixture, it follows that the likelihood-ratio distance (and consequently LRMF) can be viewed as biased estimates of JSD.
\begin{equation*}
    d_{\Lambda}(A, B) = 2\cdot\operatorname{JSD}(A, B) - \mathcal D_{KL}(A, M) - \mathcal D_{KL}(B, M) + 2\cdot \mathcal D_{KL}((A+B)/2, M)
\end{equation*}
Also, if the density family $M$ is ``convex'', in a sense that for any two densities from $M$ their equal mixture also lies in $M$,
% \forall \theta_A, \theta_B \ \exists \theta_S \ \forall x \ \ P_M(x; \theta_A) + P_M(x; \theta_B) = 2 \cdot P_M(x; \theta_S) $$
then by rearranging the terms in the definition of the likelihood-ratio distance, and noticing that the optimal shared model is the equal mixture of two densities, it becomes evident that the LRMF objective is equivalent to the GAN objective with the appropriate choice of the discriminator family:
\begin{gather*}
    \min_{T} d_{\Lambda}(T(A), B, M) = \min_{T} \max_{\theta_{AT}, \theta_B} \min_{\theta_S} \left[\log \frac{P_M(T(A); \theta_{AT})}{P_M(T(A); \theta_S)} + \log \frac{P_M(B; \theta_B)}{P_M(B; \theta_S)}\right] \\ 
    = \min_{T} \max_{\theta_{AT}, \theta_B} \left[\log \frac{P_M(T(A); \theta_{AT})}{P_M(T(A); \theta_{AT}) + P_M(T(A); \theta_B)} + \log \frac{P_M(B; \theta_B)}{P_M(B; \theta_{AT}) + P_M(B; \theta_B)} + \log 4\right] \\
    = \min_{T} \max_{D \in \mathcal H} \big[\log D(T(A)) + \log \left(1 - D(B)\right) + \log 4\big], \quad \mathcal H(\theta, \theta') = \left\{\frac{P_M(x; \theta)}{P_M(x; \theta) + P_M(x; \theta')} \right\}.
\end{gather*}
 
Since $M$ is not ``convex'' in most cases, minimizing the LRMF objective is equivalent to adversarially aligning two datasets against a regularized discriminator. From the adversarial network perspective, the reason why $\mathcal L_{\text{LRMF}}$ manages to solve this min-max problem using plain minimization is because for any flow transformation parameter $\phi$ the optimal discriminator between $T(A; \phi)$ and $B$ is defined in closed form:  $D^*(x, \phi) = P_M(x; \theta_B^*)/\big(P_M(x; \theta_B^*) + P_M(T^{-1}(x; \phi); \theta_A^*)\det|\nabla_x T^{-1}(x; \phi)|\big)$. 

\textbf{Vanishing of generator gradients.} The relation presented above suggests that the analysis performed by \citet{Arjovsky2017TowardsPM} for GANs (Theorem~2.4, page~6) applies to LRMF as well, meaning that gradients of the LRMF objective w.r.t. the learned transformation parameters might vanish in higher dimensions. This implies that while the inequality (\ref{eq:lrmf_bound}) always holds, the model produces a useful alignment only when a sufficiently ``deep'' minimum of the LRMF loss (\ref{eq:lrmf_loss}) is found, otherwise the method fails, and the loss value should be indicative of this. An example presented below shows that reaching this deep minimum becomes exponentially more difficult as the distance between distributions grows, which is often the case in higher dimensions. 

\begin{tcolorbox}
\begin{example}\label{ex:mixture_lrmf_gradient_failure}
{\small
Consider $M(\theta)$ that parameterizes all equal mixtures of two univariate Gaussians with equal variances, i.e. $\theta = (\mu_s^{(1)}, \mu_s^{(2)}, \sigma_s^2)$ and $P_M(x \ | \ \theta) = \frac{1}{2}\left(\mathcal N(x | \mu_s^{(1)}, \sigma_s^2) + \mathcal N(x | \mu_s^{(2)}, \sigma_s^2)\right)$. Consider $A$ sampled from $M(\mu+\delta, \mu-\delta, \sigma_0^2)$ and $B_{\mu}$ sampled from $M(\delta, -\delta, \sigma_0^2)$ for some fixed $\delta, \mu$ and $\sigma_0$. Let transformations be restricted to shifts $T(x; b) = x + b$, so $\phi = b$, and $\log\det|\nabla_x T(x; \phi)| = 0$, and $\mathcal E_{bias} = 0$ since $M$ can approximate both $A$ and $T(A; b)$ perfectly for any $b$. For a sufficiently large $\mu$, optimal shared model parameters can be found in closed form: $\theta^* = (\mu + b, 0, \sigma_0^2 + \delta^2)$. This way the LRMF loss can be computed in closed form up to the cross-entropy: $L(b, \mu) := \min_{\theta} \mathcal L_{\text{LRMF}}(A, B_{\mu}, b, \theta) = -2 H[M(\delta, -\delta, \sigma_0^2), M(\mu + b, 0, \sigma_0^2 + \delta^2)] + C$. A simulation provided in the supplementary Section \ref{sec:numerical_example_mixture_lrmf} shows that the norm of the gradient of the LRMF objective decays exponentially as a function of $\mu$: $\left\lVert [\partial L(b, \mu) / \partial \mu] (0, \mu) \right\rVert \propto \exp(-\mu^2)$, meaning that as $A$ and $B_{\mu}$ become further, the objective quickly becomes flat w.r.t $\phi$ near the initial $\phi_{t=0} = 0$.
}
\end{example}
\end{tcolorbox}

% \Bigr|_

\textbf{Model complexity.} We propose the following intuition: 1) chose the family $M(\theta)$ that gives highest validation likelihood on $B$, since at optimum the shared model has to approximate the true underlying $P_B$ well; 2) chose the family $T(x; \phi)$ that has fewer degrees of freedom then $M$, since otherwise the problem becomes underspecified. For example, consider $M$ containing all univariate Gaussians parameterized by two parameters $(\mu, \sigma)$ aligned using polynomial transformations of the form $T(x; a_0, a_1, a_2) = a_2 x^2 + a_1 x + a_0$. In Example \ref{ex:gaussian_lrmf} we showed that Gaussian LRMF is equivalent to moment matching for two first moments, but with this choice of $T$, there exist infinitely many solutions for $\phi$ that all produce the desired mean and variance of the transformed dataset. 
% \begin{align*}
%     \min_{T} d_{\Lambda}(T(A), B, M) &= \min_{T} \max_{\theta_{AT}, \theta_B} \min_{\theta_S} \left[\log \frac{P_M(T(A); \theta_{AT})}{P_M(A; \theta_S)} + \log \frac{P_M(B; \theta_B)}{P_M(B; \theta_S)}\right] \\ 
%     &= \min_{T} \max_{\theta_{AT}, \theta_B} \left[\log \frac{P_M(T(A); \theta_{AT})}{P_M(T(A); \theta_{AT}) + P_M(B; \theta_B)} + \log \frac{P_M(B; \theta_B)}{P_M(T(A); \theta_{AT}) + P_M(B; \theta_B)} + \log 4\right] \\
%     &= \min_{T} \max_{D} \big[\log D(T(A)) + \log \left(1 - D(B)\right) + \log 4\big]
% \end{align*}

% \textbf{Design choices.} We refer our readers to the supplementary for the discussion of the score test version of LRMF, its relation to the Invariant Risk Minimization \cite{arjovsky2019invariant}, and other alternative design choices.

% \input{3_method_new_old.tex}
\section{Related work}
In this section we summarize the prior work on addressing domain adaptation as distribution alignment, recent advances in modeling probability densities using normalizing flows, and prior attempts at applying flows to domain adaptation and distribution discrepancy estimation.

\textbf{Domain Adaptation.} \citet{ben2010theory} showed that the test error of the learning algorithm trained and tested on samples from different distributions labeled using a shared ``ground truth'' labeling function is bounded by the ${\mathcal H\Delta\mathcal H}$-distance between the two distributions, therefore framing domain adaptation as distribution alignment. This particular distance is difficult to estimate in practice, so early neural feature-level domain adaptation methods such as deep domain confusion \cite{tzeng2014deep}, DAN \cite{long2015learning} or JAN \cite{long2017deep} directly optimized estimates of non-parametric statistical distances (e.g. maximum mean discrepancy) between deep features of data points from two domains. Other early neural DA methods approximated domain distributions via simple parametric models, for example DeepCORAL \cite{sun2016deep} minimizes KL-divergence between pairs of Gaussians. Unfortunately, these approaches struggle to capture the internal structure of real-world datasets. Adversarial (GAN-based) approaches, such as ADDA \cite{ganin2016domain} and DANN \cite{tzeng2017adversarial}, address these limitations using deep convolutional domain discriminators. However, adversarial models are notoriously hard to train and provide few automated domain-agnostic convergence validation and model selection protocols, unless ground truth labels are available. Many recent improvements in the performance of classifiers adapted using adversarial alignment rely techniques utilizing source labels, such as semantic consistency loss \cite{hoffman2018cycada}, classifier discrepancy loss \cite{saito2018maximum}, or pseudo-labeling \cite{french2018selfensembling}, added on top of the unsupervised adversarial alignment. The comparison to methods that use source labels is beyond the scope of this work, since we are primarily interested in improving the robustness of the underlying alignment method.

\textbf{Normalizing Flows.} The main assumption behind normalizing flows \cite{rezende2015variational} is that the observed data can be modeled as a simple distribution transformed by an unknown invertible transformation. Then the density at a given point can be estimated using the change of variable formula by estimating the determinant of the Jacobian of that transformation at the given point. The main challenge in developing such models is to define a class of transformations that are invertible, rich enough to model real-world distributions, and simple enough to enable direct estimation of the aforementioned Jacobian determinant. Most notable examples of recently proposed normalizing flows include Real NVP \cite{dinh2016density}, GLOW \cite{kingma2018glow} built upon Real NVP with more general learnable permutations and trained at multiple scales to handle high resolution images, and the recent FFJORD \cite{grathwohl2018ffjord}, that used forward simulation of an ODE with an velocity field parameterized by a neural network as a flow transformation.

\textbf{Composition of inverted flows.} AlignFlow \cite{grover2019alignflow} is built of two flow models $G$ and $F$ trained on datasets $A$ and $B$ in the ``back-to-back'' composition $F \circ G^{-1}$ to map points from $A$ to $B$. We argue that the structure of the dataset manifold is destroyed if two flow are trained independently, since two independently learned ``foldings'' of lower-dimensional surfaces into the interior of a Gaussian ball are almost surely ``incompatible'' and render correspondences between $F^{-1}(B)$ and $G^{-1}(A)$ meaningless. \citet{grover2019alignflow} suggests to share some weights between $F$ and $G$, but we propose that this solution does not addresses the core of the issue. \citet{yang2019pointflow} showed that PointFlow - a variational FFJORD trained on point clouds of mesh surfaces - can be used to align these point clouds in the $F \circ G^{-1}$ fashion. But the point correspondences found by the PointFlow are again due to the spatial co-occurrence of respective parts of meshes (left bottom leg is always at the bottom left) and do not respect the structure of respective surface manifolds. Our approach requires 2-3 times more parameters then our composition-based baselines, but in the next section we show that it preserves the local structure of aligned domains better, and the higher number of trainable parameters does not cause overfitting.

\textbf{CycleGAN with normalizing flows.} RevGAN \cite{van2019reversible} used GLOW \cite{kingma2018glow} to enforce the cycle consistency of the CycleGAN, and left the loss and the adversarial training procedure unchanged. We believe that the normalizing flow model for dataset alignment should be trained via maximum likelihood since the ability to fit rich models with plain minimization and validate their performance on held out sets are the primary selling points of normalizing flows that should not be dismissed.

\textbf{Likelihood ratio testing for out-of-distribution detection.} \citet{nalisnick2018deep} recently observed that the average likelihood is not sufficient for determining whether the given dataset came from the same distribution as the dataset used for training the density model. A recent paper by \citet{LLR_NIPS2019_9611} suggested to use log-likelihood ratio test on LSTMs to \textit{detect} distribution discrepancy in genomic
sequences, whereas we propose a \textit{non-adversarial} procedure for \textit{minimizing} this measure of discrepancy using unique properties of normalizing flows.

% \section{Experiments and Results}

% \input{figs/fig45.tex}
% \input{figs/fig6.tex}

\section{Experiments and Results} 

In this section, we present experiments that verify that minimizing the proposed LRMF objective (\ref{eq:lrmf_loss}) with Gaussian, RealNVP, and FFJORD density estimators does indeed result in dataset alignment. We also show that both under- and over-parameterized LRMFs performed well in practice, and that resulting flows preserved the local structure of aligned datasets better then non-parametric objectives and the AlignFlow-inspired \cite{grover2019alignflow}  baseline, and were overall more stable then parametric adversarial objectives. We also show that the RealNVP LRMF produced a semantically meaningful alignment in the embedding space of an autoencoder trained simultaneously on two digit domains (MNIST and USPS) and preserved the manifold structure of one mesh surface distribution mapped to the surface distribution of a different mesh. We provide Jupyter notebooks with code in JAX \cite{jax2018github} and TensorFlow Probability (TFP) \cite{DBLP:journals/corr/abs-1711-10604}. 

\textbf{Setup 1: Moons and blobs}. We used LRMF with Gaussian, Real NVP, and FFJORD densities $P_M(x; \theta)$ with affine, NVP, and FFJORD transformations $T(x; \phi)$ respectively to align pairs of moon-shaped and blob-shaped datasets. The blobs dataset pair contains two samples of size $N=100$ from two Gaussians. The moons dataset contains two pairs of moons rotated $50^{\circ}$ relative to one another. We used original hyperparameters and network architectures from Real NVP \cite{dinh2016density} and FFJORD \cite{grathwohl2018ffjord}, the exact values are given in the supplementary. We also measured how well the learned LRMF transformation preserved the local structure of the input compared to other common minimization objectives (EMD, MMD) and the ``back-to-back'' composition of flows using a 1-nearest neighbor classifier trained on the target and evaluated on the transformed source. We also compared our objective to the adversarial network with spectral normalized discriminator (SN-GAN) in terms of how well their alignment quality can be judged based on the objective value alone.

\begin{figure}[t]
\begin{center}
\vspace{-40px}
\includegraphics[trim=0 230 0 0,width=\textwidth]{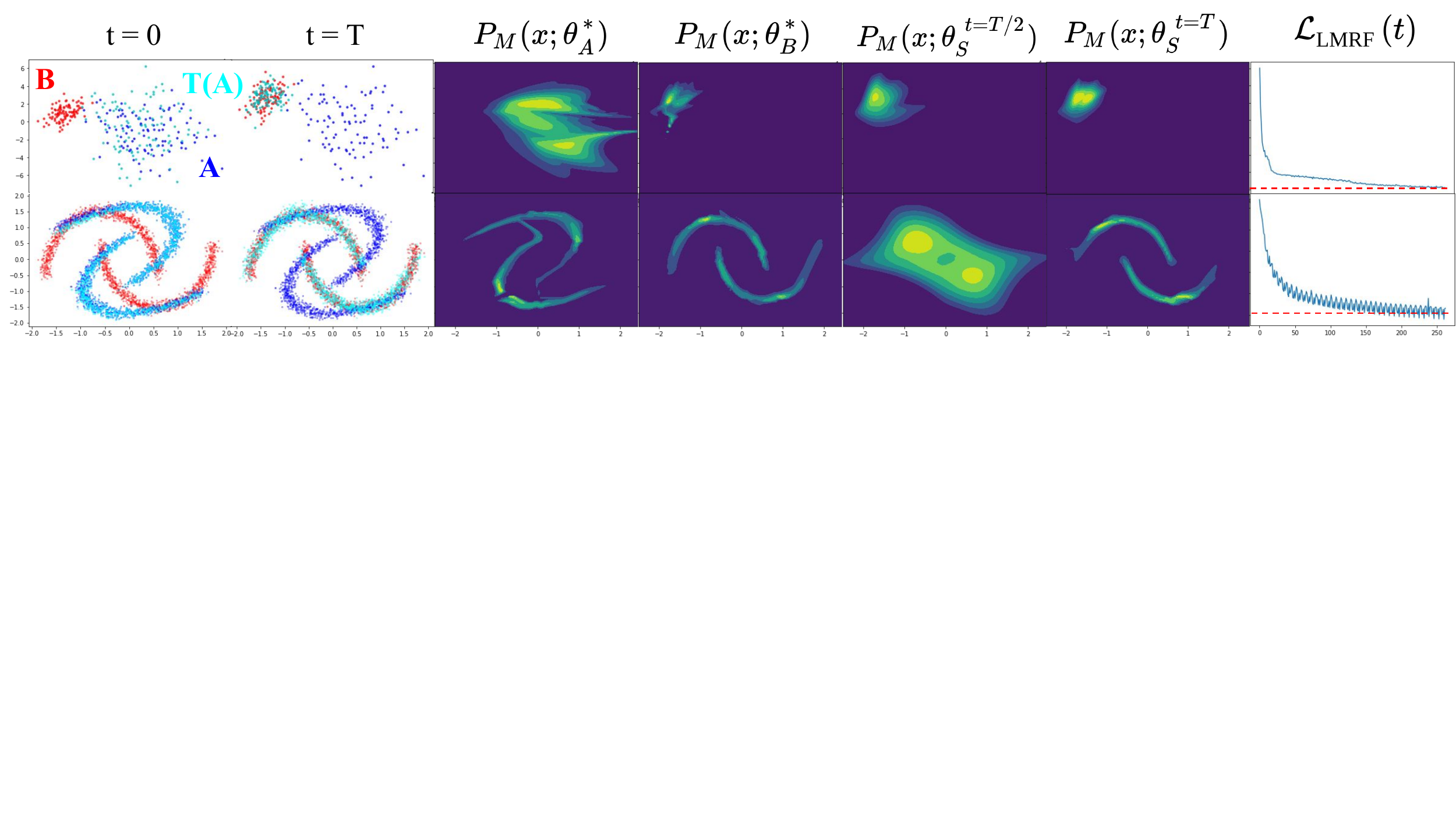}
\caption{\textbf{The dynamics of training a Real NVP LRMF on the blob (first row) and moons (second row) datasets}. Blue, red and cyan points represent $A, B$ and $T(A)$ respectively. First two columns show $T(A)$ before and after training. Third and forth columns shows optimal models from $M$ for $A$ and $B$. Fifth and sixth columns show the evolution of the shared model. The last column shows the LRMF loss over time. Even a severely \textit{overparameterized} LRMF does a good job at aligning blob distributions. The animated version that shows the evolution of respective models is available on the project web-page \href{http://ai.bu.edu/lrmf}{ai.bu.edu/lrmf}. \textit{Best viewed in color}.} \label{fig:lrmf_real_nvp_short}
\end{center}

% \bigskip

% \vspace{-9px}
% \includegraphics[trim=0 90
% 0 0,width=0.9\textwidth]{imgs/figs/gan_fig.pdf}
\includegraphics[trim=0 290 0 0,width=\textwidth]{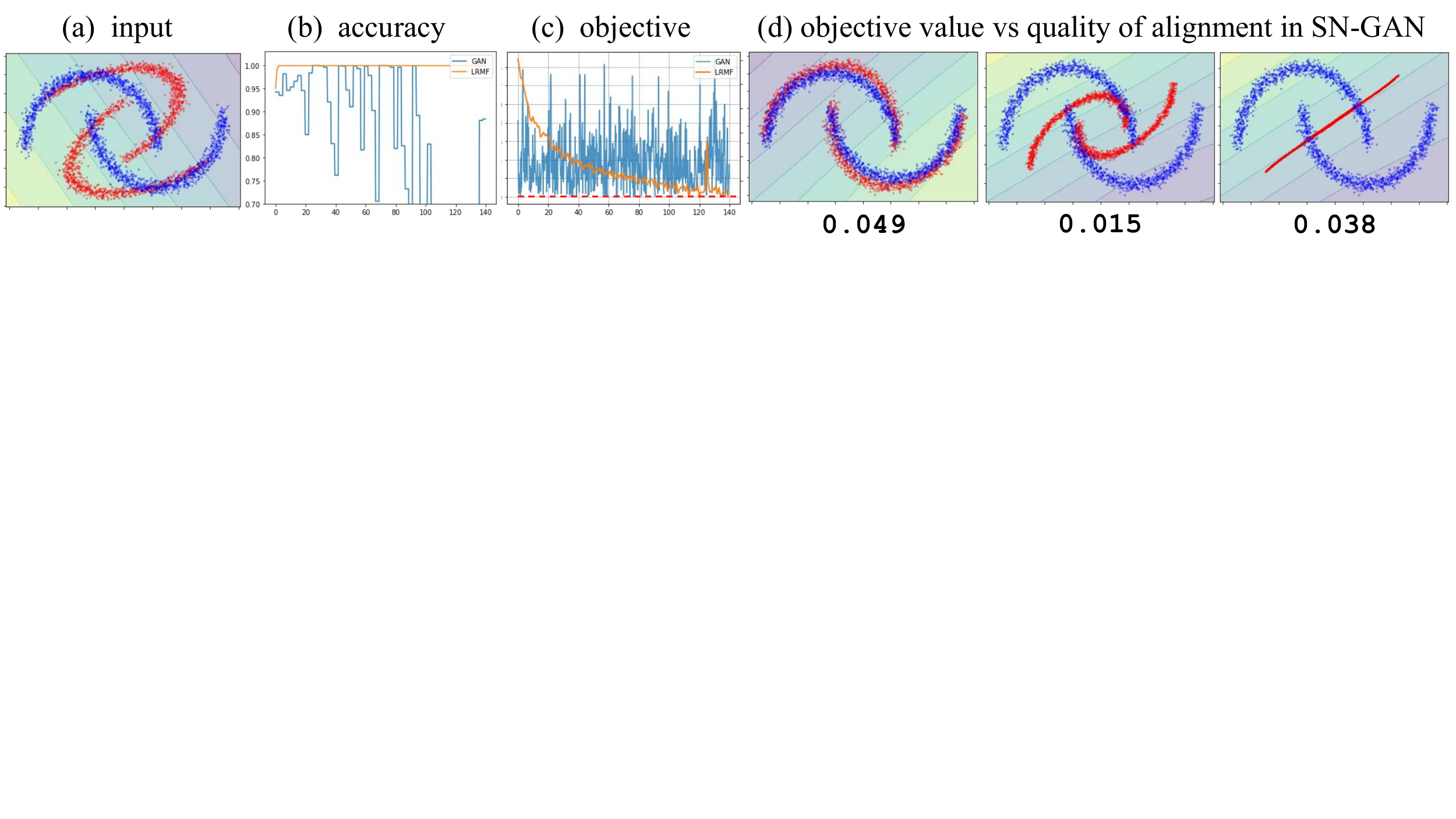}
% \vspace{3px}
\caption{\textbf{The dynamics of training a GAN with Spectral Normalization (SN-GAN) on the moons dataset}.  The adversarial framework provides means for aligning distributions against rich families of parametric discriminators, but requires the right choice of learning rate and an external early stopping criterion, because the absolute value of the adversarial objective (blue) is not indicative of the actual alignment quality even in low dimensions. The proposed LRMF method (orange) can be solved by plain minimization and converges to zero. \label{fig:sngan_moons} \vspace{-5px}}

\end{figure}

\begin{figure*}[t]
\centering
\vspace{-40px}
\includegraphics[clip,trim=50 180
50 0,width=\textwidth]{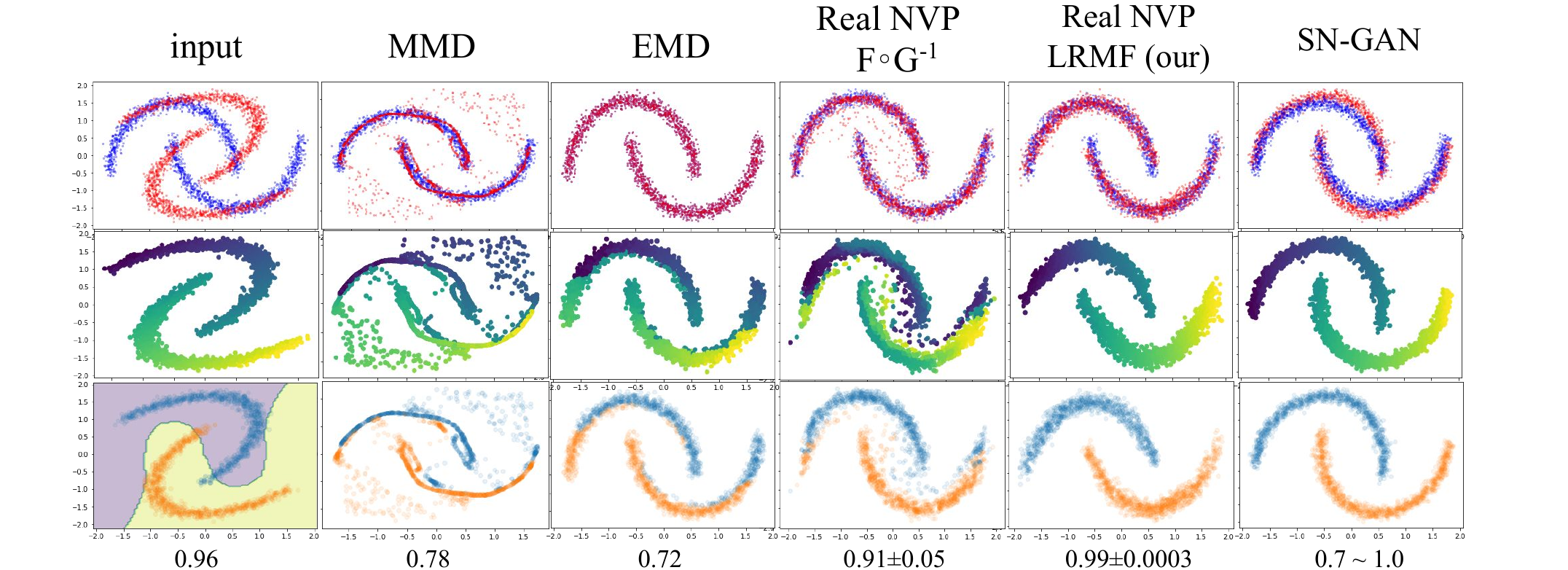}
% \vspace{-20px}
\caption{\textbf{Among the non-adversarial alignment objectives, only LRMF preserves the manifold structure of the transformed dataset}. Each domain contains two moons. The top row shows how well two domains (red and blue) are aligned by different methods trained to transform the red dataset to match the blue dataset. The middle row shows new positions of points colored consistently with the first column. The bottom row shows what happens to red moons after the alignment. Numbers at the bottom of each figure show the accuracy of the 1-nearest neighbor classifier trained on labels from the blue domain and evaluated on transformed samples from the red domain. The animated version is available on the project web-page \href{http://ai.bu.edu/lrmf}{http://ai.bu.edu/lrmf}. \label{fig:local_structure}}
% \vspace{10px}

% \bigskip
\includegraphics[clip,trim=0 260 0 0,width=\textwidth]{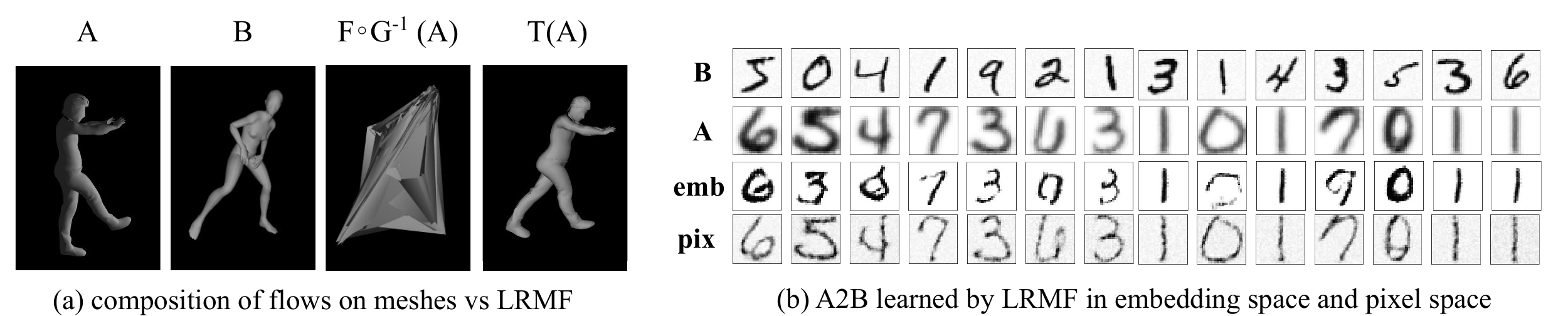}
% \includegraphics[clip,trim=0 180 0 0,width=\textwidth]{imgs/figs/fig_5_new.pdf}
% \vspace{-5px}
\caption{\textbf{RealNVP LRMF successfully semantically aligned digits and preserved the local structure of the mesh surface manifold}. \textbf{(a)} The marginal distribution produced by the ``back-to-back'' composition $F \circ G^{-1}$ of two normalizing flows trained on vertices of two meshes matches the point distribution of $B$, but the local structure of the original manifold is distorted, while LRMF preserves the local structure.  \textbf{(b)} USPS digits (\textit{B}) transformed into MNIST digits (\textit{A}) via LRMF in VAE embedding space (\textit{emb}), via LRMF in pixel space (\textit{pix}). \vspace{-20px} \label{fig:lrmf_digits_mesh}}

\end{figure*}

\textbf{Results.} 
% We parameterized the positive-definite transformation as $T(x, A, b) = A^T A\cdot x + b$ and the Gaussian density with parameters $(\mu, \Sigma^{-\frac{1}{2}})$ to ensure that $\Sigma$ is always positive definite. 
In agreement with Example~\ref{ex:gaussian_lrmf}, affine Gaussian LRMF matched first two moments of aligned distributions (Figure~\ref{fig:lrmf_gaussian}). In Real NVP (Figures \ref{fig:lrmf_real_nvp_short},\ref{fig:lrmf_real_nvp}) and FFJORD (Figure \ref{fig:lrmf_ffjord}) experiments the shared model converged to $\theta^*_B$ gradually ``enveloping'' both domains and pushing them towards each other. In both under-parameterized (Gaussian LRMF on moons) and over-parameterized (RealNVP LRMF on blobs) regime our loss successfully aligns distributions. In all experiments, the LRMF loss converged to zero in average (red line), so $\mathcal E(A, T, M) \approx 0$, meaning that affine, Real NVP, and FFJORD transformations keep input distributions ``equally far`` from $M$. The loss occasionally dropped below zero because of the variance in mini-batches.  Figure~\ref{fig:local_structure} shows that, despite good marginal alignment (top row) produced by MMD, EMD, and the $F \circ G^{-1}$ composition (inspired by AlignFlow \cite{grover2019alignflow}), the alignment produced by LRMF preserved the local structure of transformed distributions better, comparably to the SN-GAN both qualitatively (color gradients remain smooth in the middle row) and quantitatively in terms of adapted 1-NN classifier accuracy (bottom row). We believe that LRMF and SN-GAN preserved the local structure of presented datasets better than non-parametric models because assumptions about aligned distributions are too general in the non-parametric setting (overall smoothness, etc.), i.e. parametric models (flows, GANs) are better at capturing structured datasets. At the same time, Figure~\ref{fig:sngan_moons} shows that the quality of the LRMF alignment can be judged from the objective value (orange line) and stays at optima upon reaching it, while SN-GAN's performance (blue) can be hardly judged from the value of its adversarial objective and diverges even from near-optimal configurations.

\textbf{Setup 2: Meshes.} We treated vertices from two meshes as samples from two mesh surface point distributions and aligned them. After that, we draw faces of the original mesh at new vertex positions. We trained two different flows $F$ and $G$ on these surface distributions, and passed one vertex cloud through their back-to-back composition, and compared this with the result obtained using LRMF.

\textbf{Results.} Figure~\lref{fig:lrmf_digits_mesh}{a} shows that, in agreement with the previous experiment, the number of points in each sub-volume of $B$ matches the corresponding number in the transformed point cloud $F(G^{-1}(A))$, but drawing mesh faces reveals that the local structure of the original mesh surface manifold is distorted beyond recognition. The LRMF alignment (fourth column) better preserves the local structure of the original distribution - it rotated and stretched $A$ to align the most dense regions (legs, torso, head) with the most dense regions of $B$.

\begin{figure*}[t]
\begin{center}
\includegraphics[clip,trim=0 185 0 0,width=\textwidth]{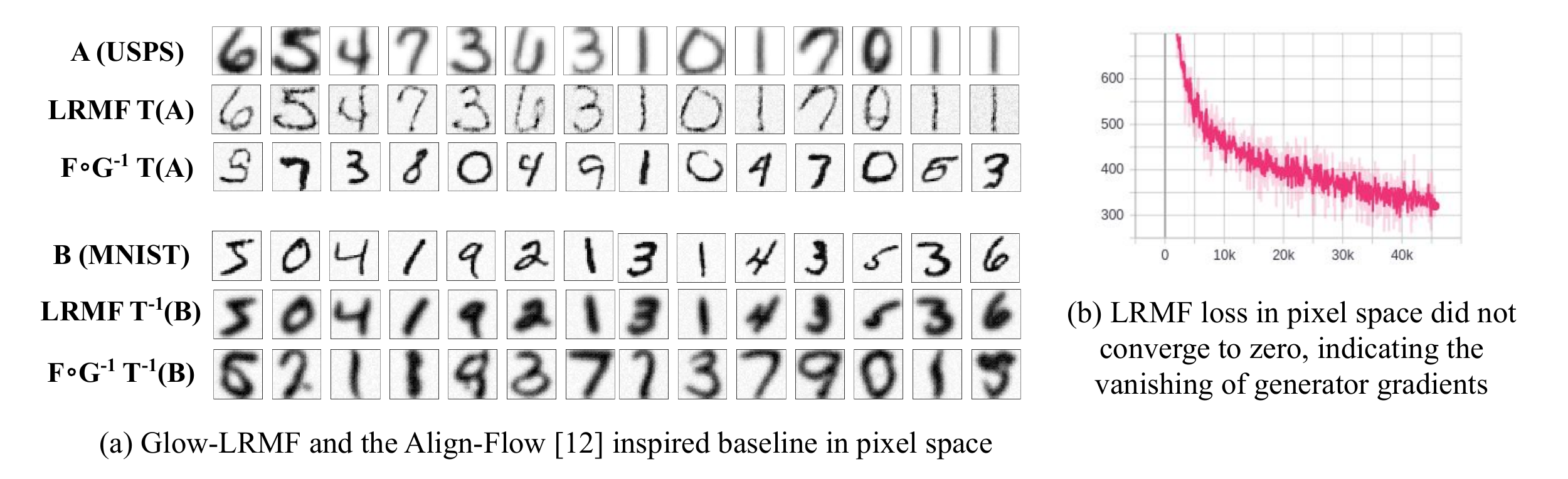}
\caption{\textbf{GLOW-LRMF did not converge in pixel space, but preserved class labels much better then the AlignFlow-inspired baseline \cite{grover2019alignflow}. } \textbf{(a)} Images generated by applying learned flow models in forward and backward direction to USPS and MNIST respectively. \textbf{(b)} GLOW-LRMF loss did not converge to zero due to the vanishing gradients in higher dimensions (pixel space). This failure mode can be detected by looking at the loss values alone. \vspace{-20px} \label{fig:lrmf_digits_only}}
\end{center}
\end{figure*}

\textbf{Setup 3: Digit embeddings}. We trained a VAE-GAN to embed unlabeled images from USPS and MNIST into a shared 32-dimensional latent space. We trained a Real NVP LRMF to map latent codes of USPS digits to latent codes of MNIST. We also trained digit label classifiers on images obtained by decoding embeddings transformed using LRMF, CORAL, and EMD and applied the McNemar test of homogeneity \cite{mcnemar1947note} to the contingency tables of prediction made by these classifiers.

\textbf{Results}. The LRMF loss attained zero. Figure~\lref{fig:lrmf_digits_mesh}{b}(emb) shows that LRMF semantically aligned images form two domains. 
% Two bar charts in the middle of Figure~\ref{fig:lrmf_digits} show log-odds of the convolutional classifier pre-trained to discriminate original MNIST images from original USPS images (on the left) and log-odds of another classifier (on the right) re-trained to discriminate original MNIST images from \textit{transformed} USPS images at each iteration. Both classifiers were able to perfectly discriminate MNIST from USPS at first (top line) and were unable to discriminate the original MNIST from the \textit{transformed} USPS. 
Classifiers trained on images transformed using LMRF had higher accuracy on the target dataset (.55 for LRMF vs .47 for EMD vs .48 for CORAL). McNemar test showed that LRMF’s improvements in accuracy were significant (p-value $\ll$ 1e-3 in all cases).

\textbf{Setup 4: GLOW.} We trained a GLOW LRMF to align USPS and MNIST in 32x32 pixel space, visualized outputs of the forward and backward transformation, and the LRMF loss value over training iterations.

\textbf{Results.} The model learned to match the stroke width across domains, but did not make images completely indistinguishable (Figure \ref{fig:lrmf_digits_only}). The shared density model converged to the local minima that corresponds to approximating $T(A)$ and $B$ as two distinct ``bubbles'' of density that fail to merge. This is the same failure mode we illustrated in Example \ref{ex:mixture_lrmf_gradient_failure} where two components of the shared model get stuck approximating datasets that are too far away, and fail to bring the model into the deeper minima. We would like to note that even though the loss did not converge to zero, i.e. the model failed to find a marginally perfect alignment, it did so \textit{not silently}, in stark contrast with adversarial methods that typically fail silently. These results agree with our hypothesis about vanishing transformation gradients in higher dimensions (end of Section 2), resulting in vast flat regions in the LRMF loss landscape with respect to the transformation parameter $\phi$, obstructing full marginal alignment. The AligFlow-inspired \cite{grover2019alignflow} composition of flows ($F\circ G^{-1}$ in Figure \ref{fig:lrmf_digits_only}), on the other hand, produced very good marginal alignment, judging from the fact that transformed images look very much like MNIST and USPS digits, but erased the semantics in the process, since there is often some mismatch present between classes of original and transformed digits.

\section{Conclusion and Future Work}
In this paper, we propose a new alignment objective parameterized by a deep density model and a normalizing flow that, when converges to zero, guarantees that the density model fitted to the transformed source dataset is optimal for the target and vice versa. We also show that the resulting model is robust to model misspecification and preserves the local structure better than other non-adversarial objectives. We showed that minimizing the proposed objective is equivalent to training a particular GAN, but is not subject to mode collapse and instability of adversarial training, however in higher dimensions, is still affected by the vanishing of generator gradients. Translating recent advances in dealing with the vanishing of generator gradients, such as instance noise regularization \cite{Arjovsky2017TowardsPM,roth2017stabilizing,sonderby2016amortised}, to the language of likelihood-ratio minimizing flows offers an interesting challenge for future research.

% \ifx\preprint\undefined
\section{Acknowledgements}
This work was partially supported by NSF award \#1724237, DARPA and Google.
\section{Broader Impact}

Many recent advances in deep learning rely heavily on large labeled datasets. Unfortunately, in many important problem domains, such as medical imaging, labeling costs and high variability of target environments, such as differences in image capturing medical equipment, prohibit widespread adoption of these novel deep image processing techniques.

Our work proposes a deep domain adaptation method that brings together verifiable convergence, as in older non-parametric methods, and meaningful priors over the structure of aligned datasets from deep adversarial alignment models. 

Of course, none of aforementioned advancements can guarantee perfect semantic alignment, therefore manual evaluation in critical applications, such as medical diagnosis, is still required. However, improved interpretability that comes from having a single minimization objective would definitely ease the adoption of such methods by making validation and model selection more straightforward, as well as reducing the chance of deploying a silently failing model due to human evaluator error.

As with the majority of deep models, our model might be susceptible to adversarial attacks by malicious agents, as well as privacy-related attacks, but properly addressing these issues, their consequences, and defence techniques goes far beyond the scope of this paper.
% \else
% \fi

% ------------

% \lipsum[1-4]
\clearpage

\bibliography{00_bib}
\bibliographystyle{plainnat}

\clearpage
\section{Supplementary Material}

\subsection{Pseudo-code for the learning algorithm}

\begin{algorithm}
\setstretch{1.4}
\caption{Mini-batch training of Log-Likelihood Ratio Minimizing Flow}
\textbf{inputs:} datasets $A$ and $B$; normalizing flow model $T(x; \phi)$; density model $P_M(x; \phi)$; learning~rate $\eta$; thresholds $(\epsilon, \varepsilon)$; batch size $b$; initial parameters values $\phi^{(0)},\theta_A^{(0)},\theta_B^{(0)}\theta_{AT}^{(0)},\theta_S^{(0)}$\;
\textbf{outputs:} convergence indicator $I_c$; weights $\phi^*$ that make $T(A; \phi^*)$ and $B$ equivalent w.r.t. M\;
\ForEach{$X \in \{A, B\}$}{
$t \leftarrow 0$ \tcp*{first learn optimal models $\theta_A^{*}, \theta_B^{*}$}
\While{$|| \nabla_{\theta} \log P_M(X; \theta_X^{(t)})|| \geq \varepsilon $}{
    $ x^{(t)} \leftarrow $ draw batch of size $b$ from $X$\;
    $\theta_X^{(t+1)} \leftarrow \theta_X^{(t)} + \eta \cdot \nabla_{\theta} \log P_M(x^{(t)}; \theta_X^{(t)}) $\;
    $t \leftarrow t + 1$;\
}
$\theta^{*}_X \leftarrow \theta_X^{(t)}$
}
$t \leftarrow 0$ \tcp*{now train LRMF}
\While{$||g_T^{(t)}|| + ||g_S^{(t)}|| \geq \varepsilon $}{
    $ A^{(t)} \leftarrow $ draw batch of size $b$ from $A$\;
    $ B^{(t)} \leftarrow $ draw batch of size $b$ from $B$\;
    % $G^{(t)} \leftarrow T(A^{(t)}; \phi^{(t)})$\;
    $g_S^{(t)} \leftarrow \nabla_{\theta} \left[ \log P_M(T(A^{(t)}; \phi^{(t)}); \theta_S^{(t)}) + \log P_M(B^{(t)}; \theta_S^{(t)}) \right] $\;
    $g_T^{(t)} \leftarrow \nabla_{\phi} \left[ \log P_M(T(A^{(t)}; \phi^{(t)}); \theta_S^{(t)}) + \log \det | \nabla_x T(A^{(t)}; \phi^{(t)}) | \right] $\;
    $\theta_S^{(t+1)} \leftarrow \theta_S^{(t)} + \eta \cdot g_S^{(t)}$\;
    $\phi^{(t+1)} \leftarrow \phi^{(t)} + \eta \cdot g_T^{(t)}$\;
    $t \leftarrow t + 1$;\
}
$\theta^{*}_S \leftarrow \theta_S^{(t)}$\;
$\phi^{*} \leftarrow \phi^{(t)}$\;
$c_{AB} \leftarrow \log P_M(A;\theta^{*}_A) + \log P_M(B;\theta^{*}_B) $ \tcp*{and check convergence}
$\mathcal L_{\text{LRMF}} \leftarrow c_{AB} - \log P_M(T(A; \phi^*); \theta_S^*) - \log P_M(B; \theta_S^*) - \log \det | \nabla_x T(A; \phi^*) |$\;
\lIf{$\mathcal L_{\text{LRMF}} \geq \epsilon$}{$I_c \leftarrow $ failed}
\lElse{$I_c \leftarrow $ succeeded}
\Return $\left(I_c, \ \phi^*\right)$
% \uIf{$\mathcal L_{\text{LRMF}} \geq \epsilon$}{
%  $I_c \leftarrow \text{failed}$ \;
% }
% \Else{
% $I_c \leftarrow \text{failed}$ \;
% }

\end{algorithm}

% \begin{lstfloat}
% \begin{lstlisting}[language=Python]
% def build_lrmf(A, B, L, flow, minimize):
%   """ Returns the LRMF objective.
  
%   A, B: datasets
%   L: (dataset, theta) -> float returns the mean negative log-likelihood of the dataset given params theta of the model from M
%   flow: (flow_L, flow_prior, flow_apply)
%     flow_L: (dataset, phi) -> float flow negative log-likelihood
%     flow_prior: dataset -> float
%     flow_apply: (dataset, phi) -> dataset
%   minimize: ((param) -> float) -> param 
%     returns the argument that minimizes the given function wrt param
%   """
%   flow_L, flow_prior, flow_apply = flow
%   theta_a = minimize(lambda th: L(A, th))
%   theta_b = minimize(lambda th: L(B, th))
%   ent_a = L(A, theta_a)
%   ent_b = L(B, theta_b)
%   prior_a = flow_prior(A)
%   const = prior_a + ent_a + ent_b
%   def lrmf_objective(theta_s, flow_phi):
%      a_t = flow_apply(A, flow_phi)
%      ent_at_s = L(a_t, theta_s)
%      ent_b_s = L(B, theta_s)
%      ent_phi = flow_L(a_t, flow_phi)
%      loss = ent_at_s + ent_b_s - ent_phi
%      return loss - const
%   return lrmf_objective

% def train_lrmf(*params):
%   """ Returns optimal LRMF parameters.
  
%   Same arguments as in `build_lrmf`.
%   """
%   lrmf_loss = build_lrmf(params)
%   _, best_phi = minimize(lrmf_loss)
%   return best_phi
% \end{lstlisting}
% \end{lstfloat}

\subsection{Attached code}

Attached IPython notebooks were tested to work as expected in Colab. The JAX version (\verb|lrmf_jax_public.ipynb|) includes experiments on 1D and 2D Gaussians and Real NVP, the Tensorflow Probabiliy (TFP) version (\verb|lrmf_tfp_public.ipynb|) includes experiments on Real NVP and FFJORD. Files \verb|vae_gan_public.ipynb, lrmf.py| and \verb|lrmf_glow_public.ipynb| contain code we used for VAE-GAN training and GLOW LRMF training

\newpage

\subsection{Hyper-parameters}
\textbf{Data.} Blobs datasets were samples from 2-dimensional Gaussians with parameters

\begin{align*}
\mu_A = \begin{bmatrix} 1.0 \\ 1.0 \end{bmatrix}, & \ \Sigma_A^{-\frac{1}{2}} = \begin{bmatrix} 0.5 & 0.7 \\ -0.5 & 0.3 \end{bmatrix} \\
\mu_B = \begin{bmatrix} 4.0 \\ -2.0 \end{bmatrix}, & \ \Sigma_B^{-\frac{1}{2}} = \begin{bmatrix} 0.5 & 3.0 \\ 3.0 & -2.0 \end{bmatrix}.
\end{align*}

The moons dataset contains two pairs of moons rotated $50^{\circ}$ relative to one another generate via \verb|sklearn.datasets.make_moons| with $\varepsilon=0.05$ containing 2000 samples each.

\textbf{Model.} 
In the affine LRMF with Gaussian density experiment (Figure~\ref{fig:lrmf_gaussian}), we parameterized the positive-definite transformation as $T(x, A, b) = A^T A\cdot x + b$ and the Gaussian density with parameters $(\mu, \Sigma^{-\frac{1}{2}})$ to ensure that $\Sigma$ is always positive definite as well. In Real NVP experiments we stacked four NVP blocks (spaced by permutations), each block parameterized by a dense neural network for predicting shift and scale with two 512-neuron hidden layers with ReLUs (the ``default'' Real NVP). In VAE-GAN experiments we trained a VAE-GAN on In FFJORD experiments we stacked two FFJORD transforms parameterized by DNN with \verb|[16, 16, 16, 2]| hidden layers with hyperbolic tangent non-linearities. For the GLOW experiment we stacked three GLOW transformations at different scales each with eight affine coupling blocks spaced by act norms and permutations each parameterized by a CNN with two hidden layers with 512 filters each. In the GLOW experiment we parameterized $T$ as the back-to-back composition of same flows used for density estimation, but initialized from scratch instead of optimal models for $A$ and $B$. We used the Adam optimizer with learning rate $10^{-5}$ for training.

\subsection{Other design considerations}

\textbf{On the relation to the Invariant Risk Minimization.}

In a recent arXiv submission, \citet{arjovsky2019invariant} suggested that in the presence of an observable variability in the environment $e$ (e.g. labeled night-vs-day variability in images) the representation function $\Phi(x)$ that minimizes the conventional empirical risk across all variations actually yields a subpar classifier. One interpretation of this statement is that instead of searching for a representation function $\Phi(x)$ that minimizes the expected value of the risk $$\mathcal R^e(f) = \mathbb E_{(X, Y) \sim P_e} l(f(X), Y)$$
across all variations in the environment $e$: $$\min_{\Phi} \min_{\theta} \mathbb E_{e} \mathcal R^e(f(\Phi(\cdot), \theta))$$
one should look for a representation that is optimal under each individual variation of the environment
\begin{align*}
    \min_{\Phi} \big[ & \min_{\theta} \mathbb E_{\epsilon} \mathcal R^e(f(\Phi(\cdot), \theta)) 
     - \mathbb E_{\epsilon} \min_{\theta_e} \mathcal R^e(f(\Phi(\cdot), \theta_e)) \ \big]
\end{align*}
\citet{arjovsky2019invariant} linearise this objective combined with the conventional ERM around the optimal $\theta$, and express the aforementioned optimally across all environments as a gradient penalty term that equals zero only if $\Phi$ is indeed optimal across all environment variations: 
$$ \min_{\Phi} \min_{\theta'} \mathbb E_{e} \mathcal R^e(f(\Phi(\cdot), \theta')) + \lambda \mathbb E_{e} || \nabla_{\theta} \mathcal R^e(f(\Phi(\cdot), \theta')) ||_2. $$

If we perform the Taylor expansion of the log-likelihood ratio statistic near the optimal shared model $\theta_S$, we get the score test statistic - a ``lighter'' version of the log-likelihood ratio test that requires training only a single model. Intuitively, if we train a model from $M$ simultaneously on two datasets $A$ and $B$ until convergence, i.e. until the average gradient of the loss w.r.t. weights $g_X = \nabla_{\theta} L(X; \theta)$ summed across both datasets becomes small $||g_A + g_B|| \leq \varepsilon$, then the combined norm of two gradients computed across each dataset independently would be small $||g_A|| + ||g_B|| \leq \varepsilon$, only under the null hypothesis ($A$ and $B$ are equivalent w.r.t. $M$). From our experience, this approach works well for detecting the presence of the domain shift, but is hardly suitable for direct minimization. 

Both procedures and resulting objectives are very much reminiscent of the log-likelihood ratio minimizing flow objective we propose in this paper, and we would have obtained the score test version if we linearized our objective around the optimal $\theta_S$. The main difference being that \citet{arjovsky2019invariant} applied the idea of invariance across changing environments to the setting of supervised training via risk minimization, whereas we apply it to unsupervised alignment via likelihood maximization.

\textbf{On directly estimating likelihood scores across domains.} 

One could suggest to estimate the similarity between datasets by directly evaluating and optimizing some combination of $P_M(A; \theta_B)$ and $P_M(B; \theta_A)$. Unfortunately, high likelihood values themselves are not very indicative of belonging to the dataset used for training the model, especially in higher dimensions, as explored by \citet{nalisnick2018deep}. One intuitive example of this effect in action is that for a high-dimensional normally distributed $x \sim \mathcal N_{d}(0, I)$ the probability of observing a sample in the neighbourhood of zero $P(||x|| \leq r)$ is small, but if we had a dataset $\{y_i\}_{i=0}^n$ sampled from that neighbourhood $||y_i|| \leq r$, its log-likelihood $\sum_i \log \mathcal N_d(y_i | 0, I)$ would be high, even higher then the likelihood of the dataset sampled from $\mathcal N_d(0, I)$ itself. The proposed method, however, is not susceptible to this issue as we always evaluate the likelihood on the same dataset we used for training.

\textbf{On matching the parameters of density models.} 

Two major objections we have to directly minimizing the distance between parameters $\theta$ of density models fitted to respective datasets $|| \theta_{AT} - \theta_B||$ are that: a) the set of parameters that describes a given distribution might be not unique, and this objective does not consider this case; and b) one would have to employ some higher-order derivatives of the likelihood function to account for the fact that not all parameters contribute equally to the learned density function, therefore rendering this objective computationally infeasible to optimize for even moderately complicated density models.

\textbf{On replacing the Gaussian prior with a learned density in normalizing flows.} 

We explored whether a similar distribution alignment effect can be achieved by directly fitting a density model to the target distribution $B$ to obtain the optimal $\theta_B^*$ first, and then fitting a flow model $T(x, \phi)$ to the dataset $A$ but replacing the Gaussian prior with the learned density of $B$:
\begin{gather*}
    \begin{aligned}
    \max_{\phi} \Big[ \log P_M(T^{-1}(A, \phi) ; \theta^*_B) - \log\det|\nabla_x T(A; \phi)|  \Big].
    \end{aligned}
\end{gather*}
While this procedure worked on distributions that were very similar to begin with, in the majority of cases the log-likelihood fit to $B$ did not provide informative gradients when evaluated on the transformed dataset, as the KL-divergence between distributions with disjoint supports is infinite. Moreover, even when this objective did not explode, multi-modality of $P_M(x; \theta_B)$ often caused the learned transformation to map $A$ to one of its modes. Training both $\phi$ and $\theta_B$ jointly or in alternation yielded a procedure that was very sensitive to the choice of learning rates and hyperparameters, and failed silently, which were the reasons we abandoned adversarial methods in the first place. The LRMF method described in this paper is not susceptible to this problem, because we never train a density estimator on one dataset and evaluate its log-likelihood on another dataset.

\subsection{FFJORD LRMF experiment on moons.}

As mentioned in the main paper, FFJORD LRMF performed on par with Real NVP version. We had to fit $T(x, \phi)$ to identity function prior to optimizing the LRMF objective, because the glorot uniform initialized 5-layer neural network with tanh non-linearities (used as a velocity field in FFJORD) generated significantly non-zero outputs. The dynamics can be found in the Figure \ref{fig:lrmf_ffjord}.

\subsection{Proof of Lemma 2.1}\label{sec:lemma21_proof}

\begin{proof}
If we define $f(x) = \log P_M(A, x)$ and $g(x) = \log P_M(B, x)$, the first statement $d_{\Lambda} \geq 0$ follows from the fact that 
\begin{gather*}
    \forall x \  f(x) + g(x) \geq \min_x f(x) + \min_x g(x) \ \Rightarrow
    \min_x (f(x) + g(x)) - \min_x f(x) - \min_x g(x) \geq 0
\end{gather*}
The second statement $f(x^*) = \min_x f(x), g(x^*) = \min_x g(x)$ comes form the fact that the equality holds only if there exists such $x^*$ that 
\begin{gather*}
    f(x^*) + g(x^*) = \min_x f(x) + \min_x g(x)
\end{gather*}
Assume that $f(x^*) \neq \min_x f(x)$, then $f(x^*) > \min_x f(x)$ from the definition of the $\min$, therefore $$g(x^*) = (f(x^*) + g(x^*)) - f(x^*) < (\min_x f(x) + \min_x g(x)) - \min_x f(x) = \min_x g(x),$$ which contradicts the definition of the $\min_x g(x)$, therefore $f(x^*) = \min_x f(x)$.
\end{proof}

\subsection{Proof of Lemma 2.2}\label{sec:lemma22_proof}
First, we add and remove the true (unknown) entropy $H[P_A] = - \mathbb E_{a \sim P_A} \log P_A(a)$:
\begin{gather*}
    \max_{\theta_{A}} \mathbb E_{a \sim P_A} \log P_M(a; \theta_{A}) = \max_{\theta_{A}} \left[ \mathbb E_{a \sim P_A} \log P_A(a) - \mathbb E_{a \sim P_A} \log \frac{P_{A}(a)}{P_M(a; \theta_{A})} \right] \\ = H[P_A] - \min_{\theta_{A}} \mathbb E_{a \sim P_A} \left[ \log \frac{P_{A}(a)}{P_M(a; \theta_{A})} \right] = H[P_A] - \min_{\theta} \mathcal D_{KL}(P_A; M(\theta)).
    \tag{$\star$}
\end{gather*}
And then add and remove the (unknown) entropy of the transformed distribution $H[T[P_A, \phi]]$. We also use the change of variable formula $T[P_A](x) = P_A(T^{-1}(x)) \cdot \det|\nabla_x T^{-1}(x)|$, and substitute the expression for $H[P_A]$ from the previous line $(\star)$:
\begin{gather*}
\begin{aligned}
    & \max_{\theta_{AT}} \log P_M(T(A; \phi); \theta_{AT}) = \max_{\theta_{AT}} \mathbb E_{a' \sim T[P_A, \phi]} \log P_M(a'; \theta_{AT}) \\
    = & \max_{\theta_{AT}} \left[ \mathbb E_{a' \sim T[P_A, \phi]} \log T[P_A](a') - \mathbb E_{a' \sim T[P_A, \phi]} \log \frac{T[P_A, \phi](a')}{P_M(a'; \theta_{AT})} \right] \\
    = & \max_{\theta_{AT}} \Big[ \mathbb E_{a \sim P_A} P_A(T^{-1}(T(a, \phi) ,\phi)) \ 
    + \\ & \qquad\qquad\qquad 
    + \log\det|\nabla_x T^{-1}(T(a, \phi), \phi)| - \mathcal D_{KL}(T[P_A, \phi]; M(\theta_{AT})) \Big] \\ 
    = & \ H[P_A] - \log\det|\nabla_x T(A, \phi)| - \min_{\theta} \mathcal D_{KL}(T[P_A, \phi]; M(\theta)) \\
    \leq & \max_{\theta_{A}} \log P_M(A; \theta_{A}) - \log\det|\nabla_x T(A, \phi)| + \mathcal E_{bias}(A, T, M).
    \end{aligned}
\end{gather*}

\subsection{Simulation results for the Example \ref{ex:mixture_lrmf_gradient_failure}}\label{sec:numerical_example_mixture_lrmf}

We approximated $| \partial H[m_1, m_2(\mu)] / \partial \mu |$, where $m_1$ and $m_2(\mu)$ are two equal mixtures of normal distributions, by computing the partial derivative using auto-differentiation in JAX. The objective was $L = \operatorname{logsumexp}(\{\log(p_i(X; \mu)) + \log 2\}_i)$, where $\log p_i(x; \mu)$ is a log probability of the mixture component from $m_2$, and $X$ is a fixed large enough (n=100k) sample from the $m_1$. Figure \ref{fig:sup_grad} shows that $\sqrt{-\log(|\partial L / \partial \mu|)}$ fits to $a\mu + b$ for $a = 0.6, b = -1.168$ with $R = 0.99996$, therefore making us believe that $\left\lVert [\partial L(b, \mu) / \partial \mu] (0, \mu) \right\rVert \propto \exp(-\mu^2)$. The code is available in \verb|lrmf_gradient_simulation.ipynb|.

\begin{figure}
\centering
% \begin{subfigure}{.3\textwidth}
%   \centering
%   \includegraphics[width=.3\linewidth]{imgs/figs/sup_grad_1.png}
%   \caption{A subfigure 1}
% \end{subfigure}%
% \begin{subfigure}{.3\textwidth}
%   \centering
%   \includegraphics[width=.3\linewidth]{imgs/figs/sup_grad_2.png}
%   \caption{A subfigure 2}
% \end{subfigure}

% \begin{subfigure}{.3\textwidth}
%   \centering
%   \includegraphics[width=.3\linewidth]{imgs/figs/sup_grad_3.png}
%   \caption{A subfigure 3}
% \end{subfigure}
\subfloat[$\sqrt{-\log(|\partial L / \partial \mu|)}$ vs $a\mu + b$]{\includegraphics[width=.3\linewidth]{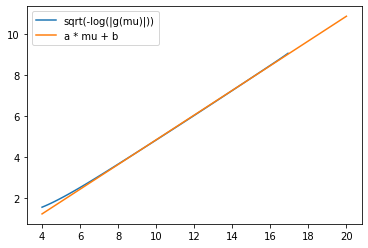}} \hspace{5px}
\subfloat[$|\partial L / \partial \mu|$ vs $\exp(-(a\mu + b)^2)$]{\includegraphics[width=.3\linewidth]{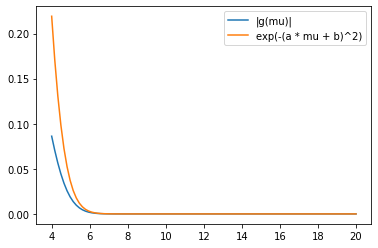}} \hspace{5px}
\subfloat[same as \textbf{b} for $\mu \in 7\dots16$] {\includegraphics[width=.3\linewidth]{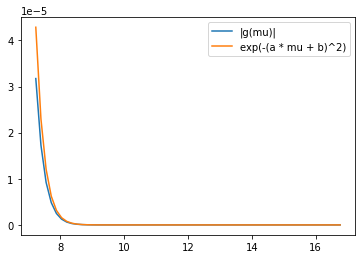}}
\caption{Gradient of the cross-entropy of between two mixture models as a function of the mean of one of the first components of the first mixture to illustrate the Example \ref{ex:mixture_lrmf_gradient_failure}, estimated using JAX.\label{fig:sup_grad}}
\end{figure}

\begin{figure*}
\centering

\vspace{-8px}
\includegraphics[trim=0 230
0 0,width=\textwidth]{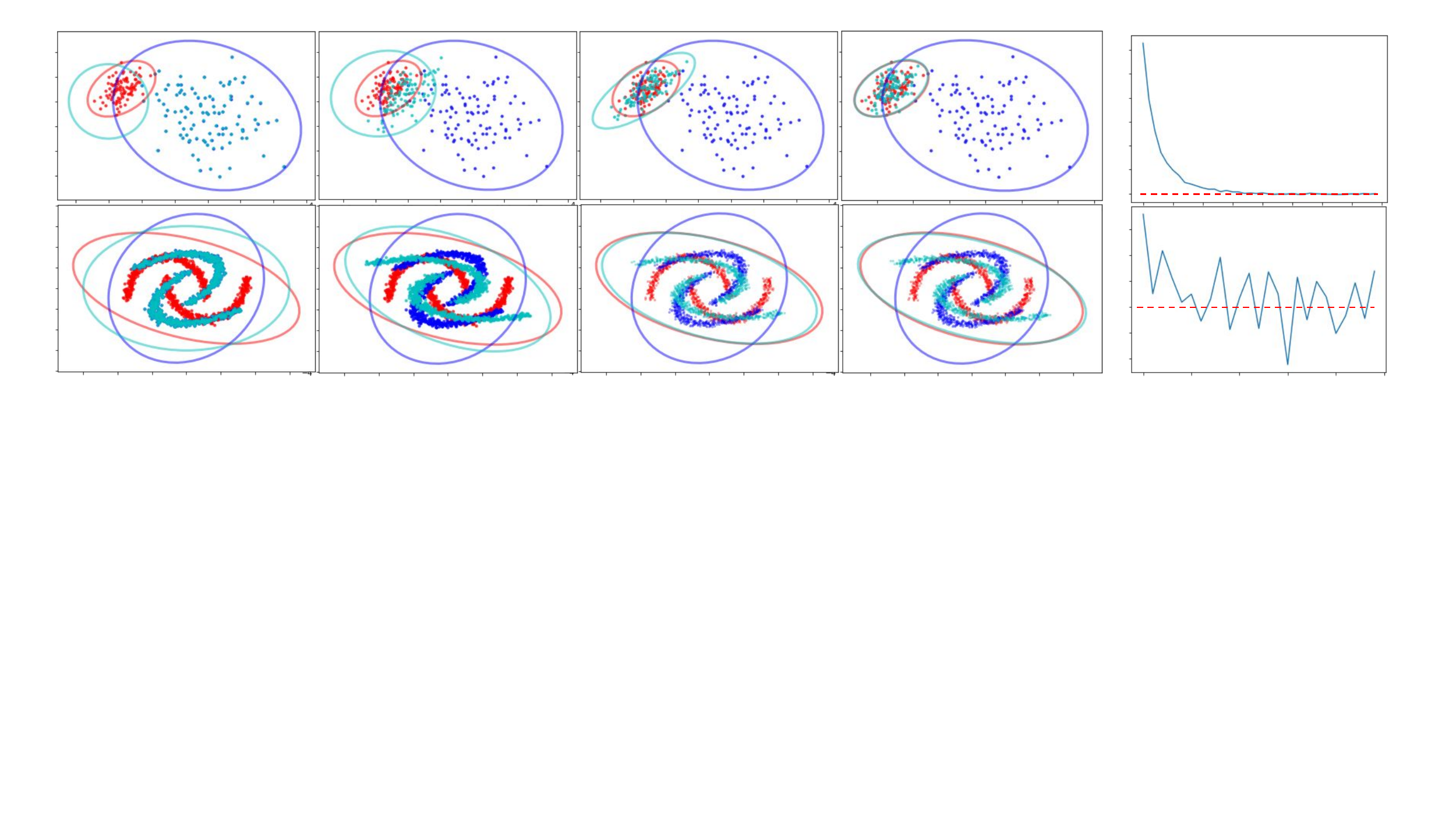}
\caption{\textbf{The dynamics of training an affine log-likelihood ratio minimizing flow (LRMF) w.r.t. the Gaussian family on the blob and moons datasets}. The LRMF is trained to match $A$ (blue) with $B$ (red), its outputs $T(A)$ are colored with cyan, circles indicate $3\sigma$ levels of $\theta_A, \theta_B$ and $\theta_{AT}$ respectively. This experiment shows that even a severely \textit{under-parameterized LRMF} does a good job at aligning distributions (second row). As in the \textbf{Example 2.1}, the optimal affine LRMF w.r.t. Gaussian family matches first two moments of given datasets. Rightmost column shows LRMF  convergence. \label{fig:lrmf_gaussian}}

\bigskip

\includegraphics[trim=0 180
0 0,width=\textwidth]{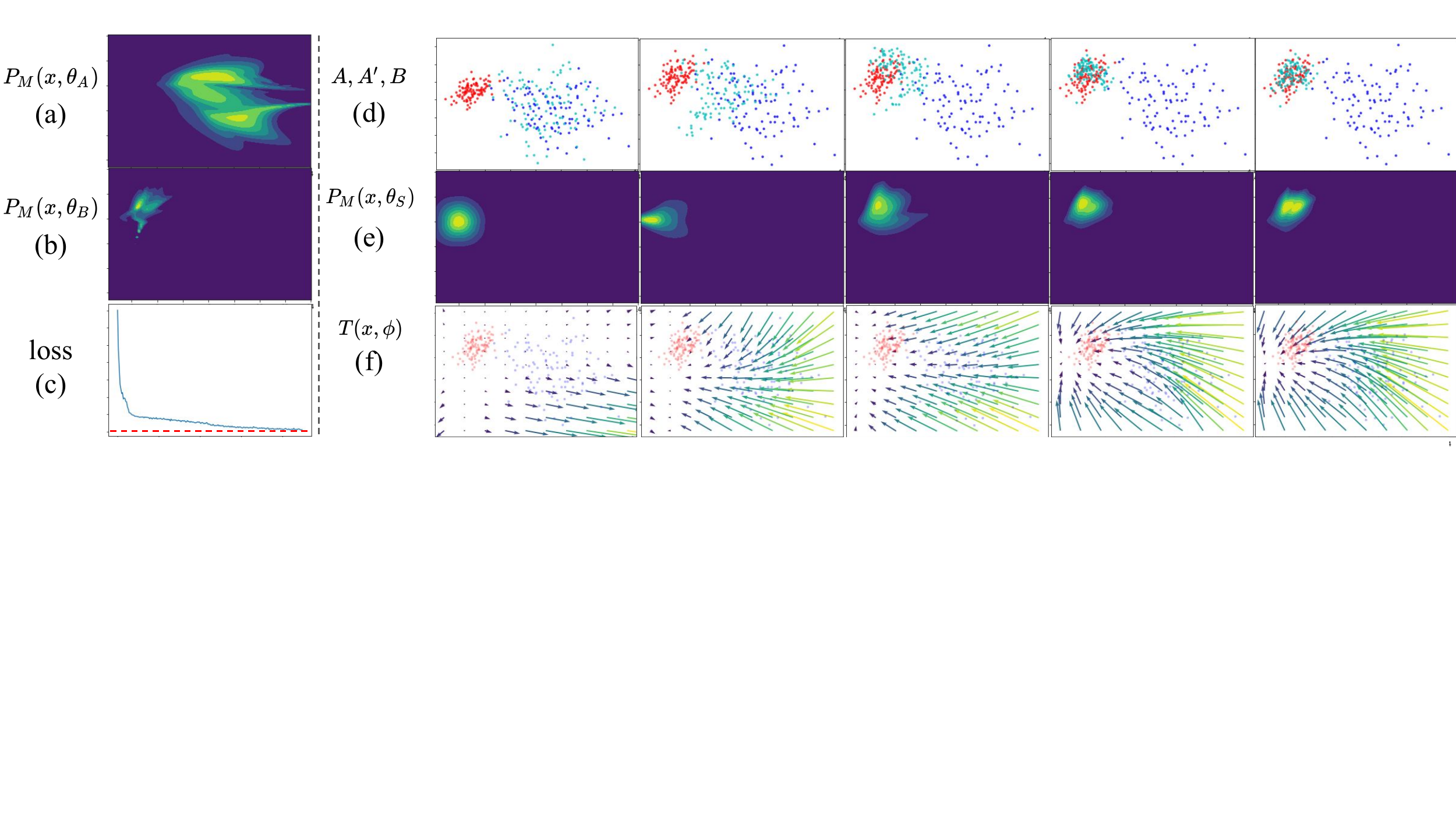}
(i) blobs

% \bigskip
\includegraphics[trim=0 180
0 0,width=\textwidth]{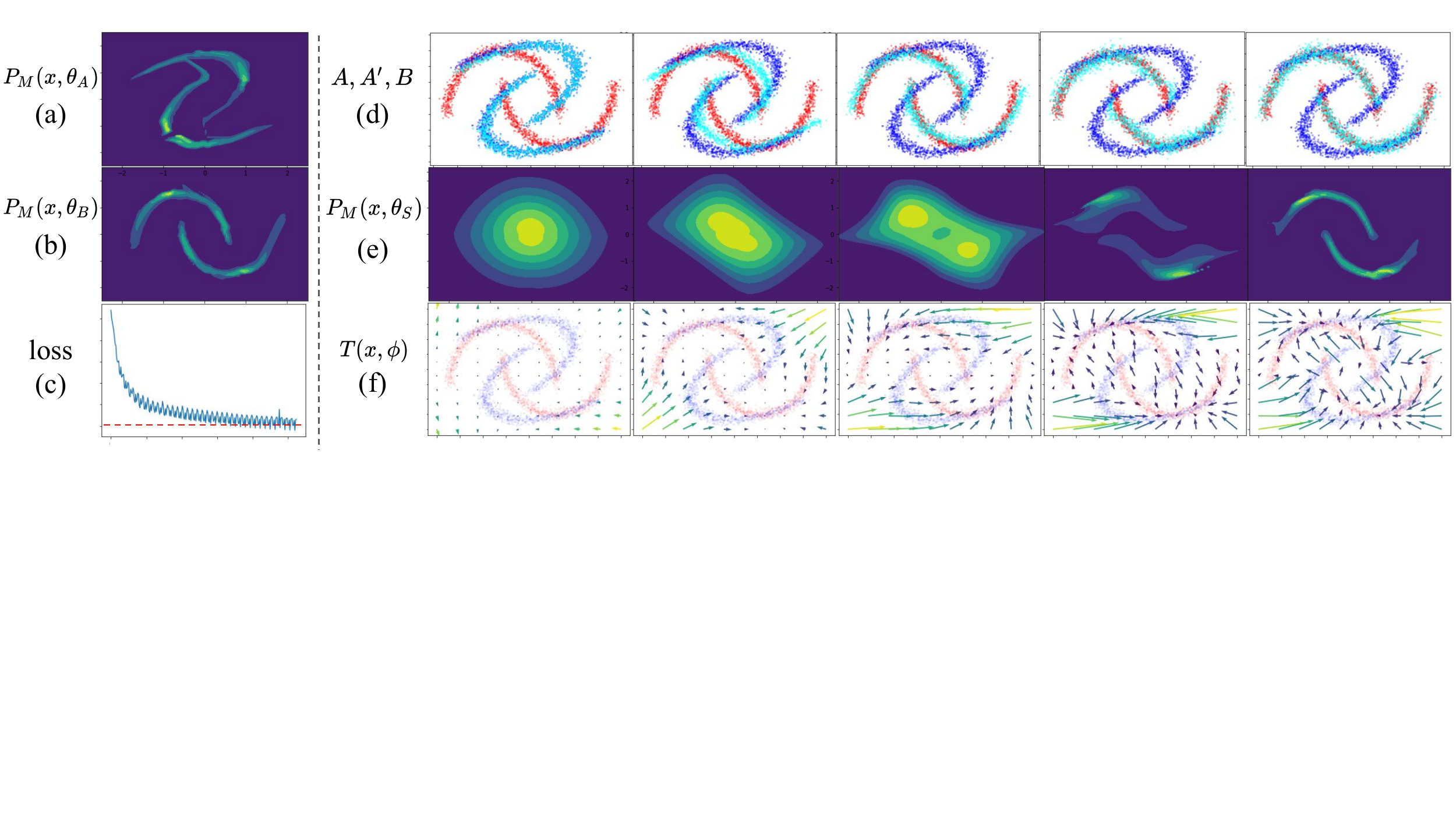}
(ii) moons
\caption{\textbf{The dynamics of training a Real NVP log-likelihood ratio minimizing flow (LRMF) on the blob and moons datasets}. This experiment shows that even a severely \textit{overparameterized LRMF} does a good job at aligning distributions: RealNVP clearly overfits to the blob dataset but learns a good alignment nevertheless.  \textbf{(a, b)} The Real NVP density estimators fitted to datasets A (blue) and B (red). \textbf{(c)} The LRMF objective (Eq 2) decreases over time and reaches zero when two datasets are aligned. The red line indicates the zero loss level. \textbf{(d)} The evolution $A$ (blue), $A' = T(A, \phi)$ (cyan) and $B$ (red). \textbf{(e)} The probability density function of the shared model $P_M(x, \theta_S)$ fitted to $A'$ and $B$. When LRMF objective converges, $P_M(x, \theta_S)$ matches $P_M(x, \theta_B)$. \textbf{(f)} The visualization of the trained normalizing flow $T$, at each point $x$ we draw a vector pointing along the direction $v = x - T(x, \phi)$ with color intensity and length proportional to $v$. \label{fig:lrmf_real_nvp}}
\end{figure*}
\begin{figure*}[ht]
\centering
\includegraphics[trim=0 50
0 0,width=\textwidth]{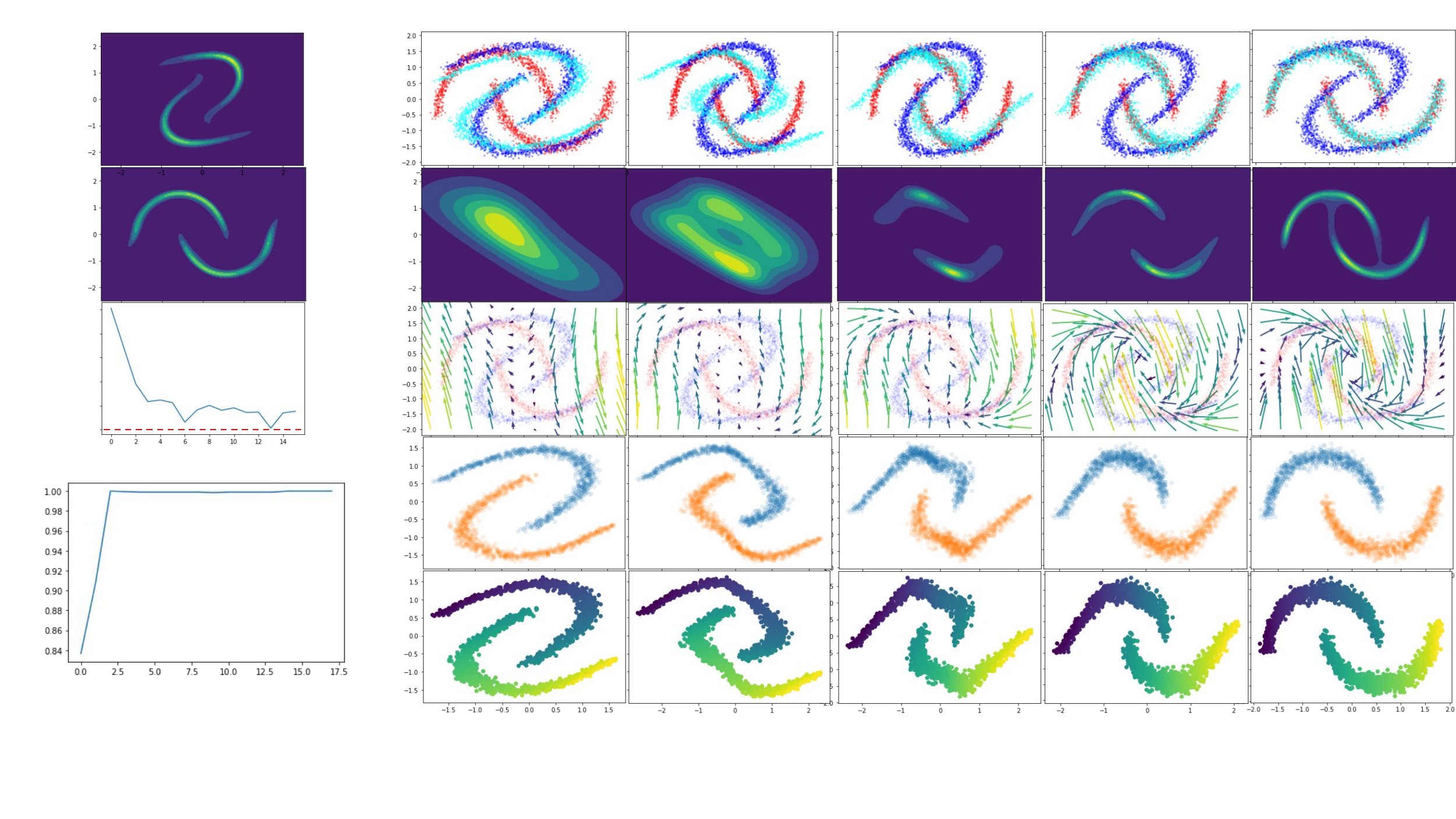}
\caption{\textbf{The dynamics of training a FFJORD log-likelihood ratio minimizing flow (LRMF) on the moons dataset}. Notations are similar to Figures 5 and 6. The left bottom plot shows changes in accuracy over time.}
\label{fig:lrmf_ffjord}
\end{figure*}
\begin{figure*}[ht]
\centering

\includegraphics[trim=0 25
20 0,width=\textwidth]{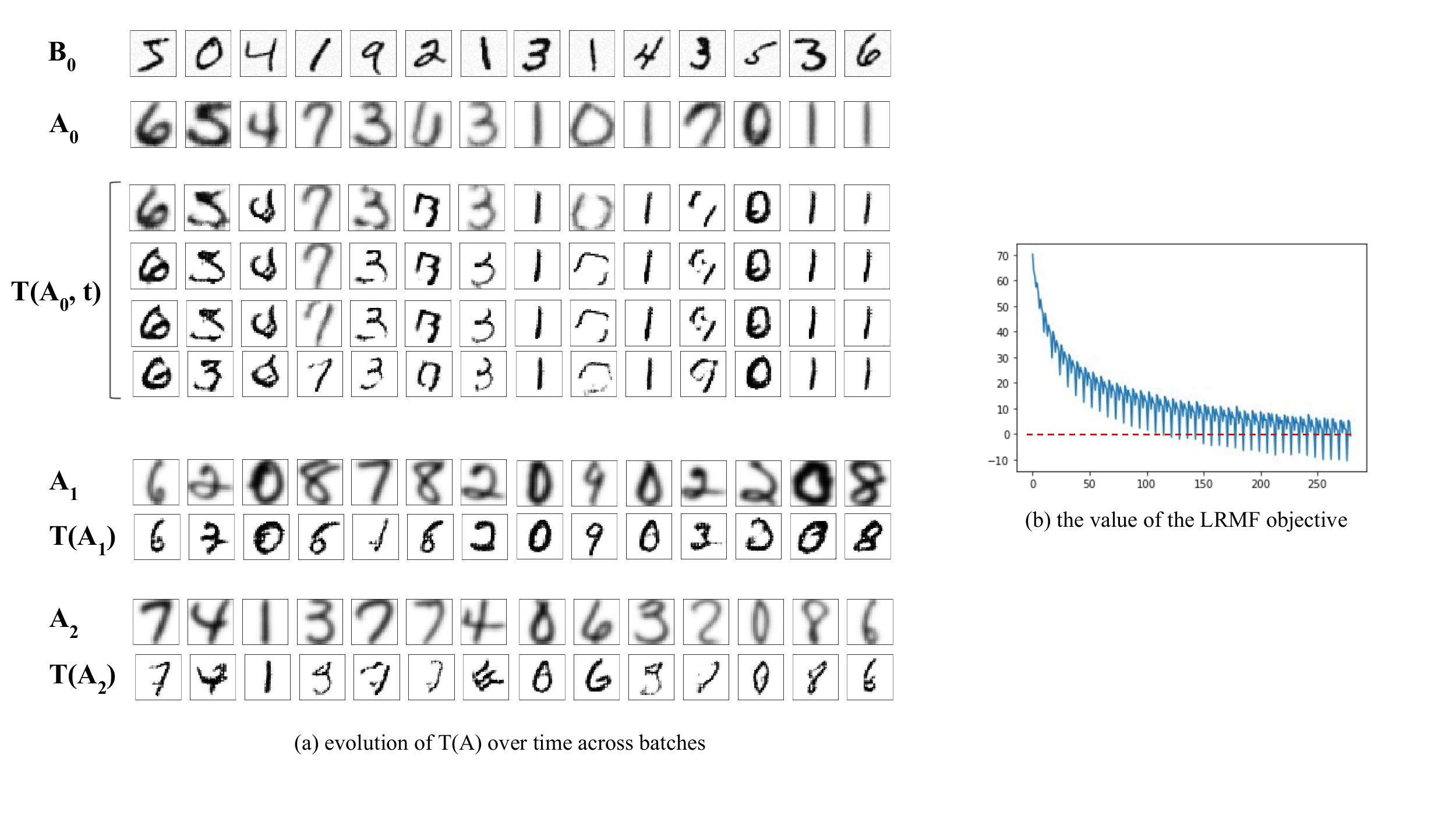}
\caption{\textbf{The dynamics of RealNVP LRMF semantically aligning USPS and MNIST digits in the latent space}. \textbf{(a)} Different rows in $T(A_0, t)$ represent transformations of the batch $A_0$ different time steps. Other rows represent the final learned transformation applied to other batches $A_1, A_2$. \textbf{(b)} The LRMF objective converged to zero in average.}
\label{fig:lrmf_digits_full}
\end{figure*}

\end{document}